\documentclass[11pt]{article}
\usepackage{rotating,amsmath,amssymb,amsfonts,amsthm}
\usepackage[dvipsnames]{xcolor}
\usepackage{subcaption}
\usepackage{epsfig,mathrsfs,natbib,color,graphics,comment}
\usepackage{multirow}
\usepackage{amsmath}
\usepackage{mathtools}
\usepackage{graphicx}
\usepackage{array,wrapfig}
\usepackage{hyperref}
\usepackage[capitalize,noabbrev]{cleveref}
\usepackage{booktabs}
\def\boxit#1{\vbox{\hrule\hbox{\vrule\kern6pt \vbox{\kern6pt#1\kern5pt}
\kern6pt\vrule}\hrule}}

\usepackage{geometry}
\usepackage{algorithm}
\usepackage{algorithmic}
\newtheorem{theorem}{Theorem}[]
\newtheorem{lemma}{Lemma}[]

\newtheorem{proposition}{Proposition}[]

\newtheorem{remark}{Remark}[]
\newtheorem{assumption}{Assumption}
\DeclareMathOperator*{\esssup}{ess\,sup}

\newcommand{\bgamma}{{\boldsymbol \gamma}}

\usepackage{authblk}

\author[1]{Sehwan Kim}
\author[1,2]{Rui Wang}
\author[3]{Wenbin Lu\thanks{Correspondence author: Wenbin Lu, email: wlu4@ncsu.edu}}

\affil[1]{Department of Population Medicine, Harvard Pilgrim Health Care Institute and Harvard Medical School, Boston, MA }
\affil[2]{Department of Biostatistics, Harvard School of Public Health, Boston, MA 
}
\affil[3]{Department of Statistics, North Carolina State University, Raleigh, NC}
\begin{document}

\title{Self-Consistent Equation-guided Neural Networks for Censored Time-to-Event Data }

\maketitle

\begin{abstract}

In survival analysis, estimating the conditional survival function given predictors is often of interest. There is a growing trend in the development of deep learning methods for analyzing censored time-to-event data, especially when dealing with high-dimensional predictors that are complexly interrelated. Many existing deep learning approaches for estimating the conditional survival functions extend the Cox regression models by replacing the linear function of predictor effects by a shallow feed-forward neural network while maintaining the proportional hazards assumption. Their implementation can be computationally intensive due to the use of the full dataset at each iteration because the use of batch data may distort the at-risk set of the partial likelihood function. To overcome these limitations, we propose a novel deep learning approach to non-parametric estimation of the conditional survival functions using the generative adversarial networks leveraging self-consistent equations. The proposed method is model-free and does not require any parametric assumptions on the structure of the conditional survival function. We establish the convergence rate of our proposed estimator of the conditional survival function. In addition, we evaluate the performance of the proposed method through simulation studies and demonstrate its application on a real-world dataset.

\end{abstract}

\section{Introduction}

    Censored time-to-event data are widely encountered in various fields where understanding the timing of events, such as failure rates or disease progression, is critical, but the exact event times may be partially observed or incomplete. For example, estimating survival probability based on covariate information is essential for risk prediction, which plays a key role in developing and evaluating personalized medicine.

    The Kaplan-Meier (KM) estimator \citep{KM1958KM}, Cox proportional hazards model \citep{cox1972coxph}, and random survival forests \citep{Hemant2008RSF} are commonly-used methods for estimating survival functions. The KM estimator is a non-parametric method suitable for population-level analyses. However, its utility is limited when the objective is to estimate conditional survival probabilities at the individual level. The Cox proportional hazards model offers a semi-parametric approach for estimating conditional survival functions, accommodating the incorporation of covariates. However, violations of the proportional hazards assumption may lead to biased parameter estimates. In contrast, random survival forests provide a non-parametric machine learning alternative that does not depend on this assumption. However, this method tends to overfit by favoring variables with a larger number of unique values, as these provide more potential split points. This increases their likelihood of being selected, even when their predictive contribution is modest, potentially distorting variable importance estimates \citep{strobl2007biasrsf}. This issue can be mitigated by using maximally selected rank statistics as the splitting criterion, albeit at the cost of increased computational runtime \citep{Wright2017splitrsf}.
    
    % {\color{blue} Need confirmation from Dr. Wang:  However, the method tends to overfit by favoring variables with a greater number of unique values, which have more potential split points. This increases their chances of being selected, even if their predictive contribution is insignificant, potentially distorting estimates of variable importance \citep{strobl2007biasrsf,Wright2017splitrsf}.}
    % {\color{red} Sehwan, this is not an accurate summary of what is in the cited papers. The so-called "split variable selection bias" is a problem with the standard split criterion. Wright et al. (2017) and others have already addressed this problem? } \textcolor{blue}{a trade-off between unbiased split variable selection and runtime.?}

    Integrating deep learning (DL) approaches into survival analysis has led to several methodological advances. These DL-based survival function estimation methods can be categorized by model class, loss function, and parametrization, which are closely interrelated. For clarity, below we provide a review of these methods divided into three main categories: (1) parametric models, (2) discrete-time models, and (3) Cox model-based methods. For a comprehensive review of these approaches, see \citep{Wiegrebe2024review}.

    First, conditional parametric models postulate the survival time as a function of covariates with an error term, typically following a Weibull or Log-normal distribution. Neural networks are used to estimate the parameters associated with the covariate effects \citep{Bennis2020wei, ava2020lognormal}. Another approach involves joint modeling survival times and covariates, assuming that survival times follow a Weibull distribution with parameters dependent on latent variables drawn from a process structured through deep exponential families (DEF) \citep{Ranganath2016deepsurv}. Second, discrete-time models estimate the conditional survival probability at each time interval, often constructed by observed event times, using neural networks, jointly with the censoring indicator \citep{Lee2018deephit} or modeling the conditional mortality function \citep{Gensheimer2018ASD}. Lastly, Cox model-based methods (e.g., \textit{DeepSurv}) use neural networks to model covariate effects on the conditional hazard function and minimize the corresponding partial likelihood for parameter estimation \citep{Katzman2018dsurv}.

    All methods within these three categories, either explicitly or implicitly, assume a specific structure for the survival function or data, such as a particular parametrization of the data-generating process for failure times, proportional hazards, or discrete time, regardless of whether deep learning approaches are applied. Moreover, Cox model-based approaches face a limitation due to the dependency of an individual's partial likelihood on risk sets. This complicates the use of stochastic gradient descent (SGD), as it necessitates using the entire dataset for computing gradient each step, increasing the computational burden. Despite potential mitigation with large batches suggested by \cite{kvamme2019coxtime}, the inherent challenges with batch size and learning efficiency still remain. 
    
    In this work, we propose a novel deep learning-based non-parametric method for estimating conditional survival functions, termed SCENE (\underline{S}elf-\underline{C}onsistent \underline{E}quation-Guided \underline{Ne}ural Net), which uses generative neural networks leveraging self-consistent equations.  The proposed method is model-free and does not impose any parametric assumptions on the structure of the conditional hazard function.

     % In this work, we propose a novel deep learning-based non-parametric method, called SCENE(\underline{S}elf-\underline{C}onsistent \underline{E}quation-Guided \underline{Ne}ural Net (SCENE)) for estimating conditional survival functions using generative neural networks that leverage self-consistent equations.{\color{red} Introduce SCENE earlier.} The proposed method is model-free and does not impose any parametric assumptions on the structure of the conditional hazard function. 

     Our contributions can be summarized in four aspects:
    % {\color{red} this seems to be a mixture of how our approach works and contributions. separate? } 
     \begin{itemize} 
 \item    
    We generalize the self-consistent equation for the KM estimator to the estimation of the conditional survival function by introducing a class of infinitely many self-consistent equations that can uniquely determine the true conditional survival function.

    \item  We develop a framework that solves the class of infinitely many self-consistent equations using a min-max optimization, inspired by the Generative Adversarial Network (GAN). Our contribution lies in leveraging a generative neural network to generate survival times for computing the conditional survival function, while employing a discriminative neural network to flexibly represent the weight functions in the self-consistent equations.

    \item  We propose a method for calculating variable importance measures to identify key predictors for building the neural networks, which can be easily incorporated into the proposed min-max optimization framework. Incorporating variable selection based on these measures enhances both the accuracy and interpretability of the resulting neural network-based estimator of the conditional survival function, particularly in high-dimensional settings.

    \item  We establish the convergence rate for the proposed neural network-based estimator of the conditional survival function.

\end{itemize}

%\paragraph{Related Works} Apart from the deep learning approaches discussed earlier, there are methods that employ generative models for survival analysis. The \textit{Deep Adversarial Time-to-Event (DATE)} primarily trains a conditional generator using a GAN \citep{fellow2014gan} on a constrained set, which includes only non-censored data, while imposing adjustments to handle censored cases \citep{chapfuwa18a}. \textit{SurvivalGAN} further extends this approach by generating the joint distribution of covariates and observed time, conditioned on the censoring indicator \citep{norcliffe23survivalgan}. Additionally, \citet{zhou2022deep} propose a conditional generator that generates the joint distribution of observed survival times and censoring indicators, given covariate information, and subsequently applies the KM estimator or the Nelson-Aalen estimator \citep{Nelson1969,Nelson1972,Aalen1978}. All of these methods utilize the GAN framework, where the loss function is based on the inverse probability metric to train the conditional generator. In contrast, our approach introduces a novel training method that leverages self-consistent equations, which is a completely new way to train the conditional generator.

%SurvivalGAN: (x^{fake},\tilde{T}^{fake}|\delta)
%Zhou: $(\tilde{T}^{\text{fake}}, \delta | X)$

The remainder of the paper is organized as follows. In Section 2, we generalize the self-consistent equation for the KM estimator to self-consistent equations for the conditional survival function. 
Section 3 reformulates the problem of solving self-consistent equations into a min-max optimization problem by expressing the task as solving the expectation of weighted self-consistent equations.
%Section 3 reformulates the problem of solving these equations into a min-max optimization framework. {\color{blue} This reformulation involves two steps: first, expanding the task of solving the self-consistent equation for each conditional survival function into solving the expectation of weighted self-consistent equations; and second, transforming the problem of solving the expectation of weighted self-consistent equations into a min-max optimization problem.} {\color{red} need to revise. Does the "first..." part belong to Section 2? the "second..." part just repeats the sentence earlier without adding any new insights? }
Section 4 describes an efficient computing algorithm for the proposed min-max optimization on function classes constructed by a conditional distribution generator and neural networks. The convergence rate for the proposed neural network-based estimator of the
conditional survival function is presented. It also discusses methods for computing variable importance within neural network models and incorporates these variable importance measures during training to enhance the accuracy of the proposed estimator. Section 5 presents simulation studies evaluating the performance of our proposed estimator compared to existing methods. 
In Section 6, we apply SCENE to real-world data to estimate the conditional breast cancer survival probabilities based on a range of covariates including gene expressions and clinical features. Section 7 concludes the paper with a discussion. Proofs of Propositions and Theorems in the main paper are given in the supplementary material. 

\section{Self-Consistent Equations for Conditional Survival Functions}

    In this section, we define the problem setting, introduce the notation used throughout the paper, and generalize the self-consistent equation for the KM estimator to self-consistent equations for conditional survival functions.
    
\subsection{Notation and formulation}

    Let $T_i$, $C_i$, and $X_i\in \mathbb{R}^p$, $i=1,\dots,N$, be $N$ independent copies of random variables $T$, $C$, and $X$, representing the survival time, censoring time, and covariates for individual $i$, respectively. As usual, we assume conditional independent censoring, that is, $T$ and $C$ are independent given covariates $X$. %{\color{green} We assume $T$ and $C$ are conditionally independent given $X$.} {\color{red} is this enough? for section 2.2, we need $T$ and $C$ are independent. Should we mention this here or there?} 
    The observed time is $\tilde{T}_i = \min(T_i, C_i)$, accompanied by a censoring indicator $\Delta_i = I(T_i \leq C_i)$. We denote the true conditional survival functions for the survival time and censoring time as $S_T^*(t|x)=P(T>t|X =x)$ and $S_C^*(t|x)=P(C>t|X=x)$, respectively, and the population level survival function $S^*(t)=P(T>t)=E[S_T^*(t|X)]$.  Our objective is to obtain estimates for the true conditional survival function $S_T^{*}(t|x)$ given a dataset $\mathcal{D} = \{\tilde{t}_i, \delta_i, x_i\}_{i=1}^N$, where $\tilde{t}_i$, $\delta_i$, and $x_i$ represent the observed realizations of the random variables $\tilde{T}_i$, $\Delta_i$, and $X_i$, respectively.
    
    % For notation simplicity, we denote $S:=S_T(\cdot|\cdot)$ throughout the paper.  

\subsection{Self-consistent equation for the KM estimator}

      %The KM estimator is a non-parametric estimator for the survival function $S^*(t)$. 
     %{\color{green} Suppose $\{T_i\}_{i=1}^N$ are independent and identically distributed (i.i.d.) random variables representing survival times, and $\{C_i\}_{i=1}^N$ are i.i.d. random variables representing censoring times. Survival and censoring times are independent, with $T_i \perp C_i$ and $\Delta_i = I(T_i \leq C_i)$, for $i = 1, \dots, N$, and we observe $\tilde{T}_i =min(T_i,C_i)$.} {\color{red} why do we need to introduce these notations twice?} 
     Given $\{\tilde{t}_i, \delta_i\}_{i=1}^N$, it is well known that the non-parametric KM estimator $\hat{S}(t)$ satisfies the self-consistent equation \citep{efron1967two}:
    \begin{equation}\label{kmselfeq2} \begin{split} S(t) &= \frac{1}{N} \sum_{i=1}^N \left\{I(\tilde{t}_i > t) + (1 - \delta_i) \frac{I(\tilde{t}_i \leq t)}{S(\tilde{t}_i)} S(t)\right\}, \end{split} \end{equation} for all $t$. Deriving the empirical self-consistent equations for the conditional survival function directly from Equation (\ref{kmselfeq2}) is challenging in part due to the presence of indicator functions. Therefore, we consider its limiting form as described below. 

    % {\color{red} To Dr. Wang:
    % \[
    % E_{T,C}[\frac{I(T>C)I(min(T,C)\leq t)}{S(min(T,C))}]=E_{T,C}[\frac{I(T>C)I(C\leq t)}{S(C)}]
    % \]
    % }
    
    As $N \to \infty$, equation (\ref{kmselfeq2}) converges to its limiting form as follows:
    \begin{equation}\label{kmselfeq} \begin{split} S(t) &= E_{\tilde{T}}\{I(\tilde{T} > t)\} + E_C\left\{\frac{S^*(C)}{S(C)}I(C \leq t)\right\} S(t). \end{split} \end{equation}
    
    Note that true survival function $S^*(t)$ satisfies Equation (\ref{kmselfeq}) for all $t \ge 0$. In the following subsection, we generalize this population-level self-consistent equation (\ref{kmselfeq}) to self-consistent equations for the conditional survival function. 
    
    % estimating the conditional expectations in the right-hand side of the equations}. {\color{red} Equation (1) does not involve expectations?  } {\color{blue} constructing estimator in the right-hand side of the equations for the conditional expectations}
    
    % We revisit the construction of the self-consistent equation. In the absence of censoring, $S^*(t)=P(T_i > t)$ can be easily estimated by $\hat{S}(t) = \frac{1}{N} \sum_{i=1}^N I(T_i > t)$. However, in the presence of censoring, where we observe $\tilde{T}_i$ instead of $T_i$ for some $i$, direct evaluation of $I(T_i > t)$ is not possible. In such cases, we substitute $I(T_i > t)$ with its conditional expectation $E[I(T_i > t) | \tilde{T}_i, \delta_i]$, leveraging the fact that $ S^*(t) = E[E[I(T_i > t | \tilde{T}_i, \delta_i)]]$, where $ E[I(T_i > t) | \tilde{T}_i, \delta_i] = I(\tilde{T}_i > t) + (1 - \delta_i) \frac{I(\tilde{T}_i \leq t)}{S^*(\tilde{T}_i)} S^*(t)$.  So, the true survival function $S^*(t)$ should satisfy

\subsection{Self-consistent equations for conditional survival functions}\label{subsection:sef-consistent}

    To construct self-consistent equations for the conditional survival function, we naturally replace $S(\cdot)$ and $S^*(\cdot)$ with the corresponding conditional survival functions and replace the expectation with the conditional expectation given $x$. Specifically, we obtain a self-consistent equation for the conditioanal survival function for each $x \in \mathcal{X}$ in the following limiting form:
    \begin{equation}\label{const:2}
    \begin{split}
        S_T(t|x) &= E_{\tilde{T}}[I(\tilde{T}>t)|X=x] + E_C\left[\frac{S_T^*(C|x)}{S_{T}(C|x)} I(C \leq t) \bigg|X= x\right] S_T(t|x).
    \end{split}
    \end{equation}
    %where the empirical version of self-consistent equations, which depends on the dataset $\mathcal{D}$, will be discussed in Section \ref{subsection:minmax}. 

    % \begin{remark}  Deriving the empirical self-consistent equations for the conditional survival function directly from Equation (\ref{kmselfeq2}) is challenging. The primary difficulty lies in {\color{green} estimating the conditional expectations in the right-hand side of the equations}. {\color{red} Equation (1) does not involve expectations?  } {\color{blue} constructing estimator in the right-hand side of the equations for the conditional expectations}
    % While it may be possible to construct kernel estimators for the conditional expectations self-consistent equations, this approach is unlikely to work in high-dimensional settings due to the curse of dimensionality. In the next Section, we will introduce an alternative representation of a class of weighted self-consistent equations, which enable the empirical estimation via min-max optimization.
    % \end{remark}
    
    % \[
    % \begin{split} 
    % S_T^*(t|x) &= S_{C}^*(t|x) S_{T}^*(t|x) + (1 - S_{C}^*(t|x)) S_{T}^*(t|x) \\
    %           &= E_{\tilde{T}}[I(\tilde{T}>t)|x] + E_C\left[\frac{S_T^*(C|x)}{S_{T}^*(C|x)} I(C \leq t) \bigg| x\right] S_{T}^*(t|x). 
    % \end{split} 
% \]

    Next, we show that given $X=x$, the true conditional survival function $S_T^*(\cdot|x)$ is the unique solution to Equation (\ref{const:2}). %Specifically, if $S_T(t|x)$ satisfies the Equation (\ref{const:2}) for all $t\in \mathcal{T}\subseteq \mathbb{R}_{\geq 0}$, then $S_T(t|x)=S_T^*(t|x)$ for $t\in \mathcal{T}$ almost surely.
    
\begin{proposition}[Uniqueness of solution]\label{uniquesol} Given $X = x\in \mathcal{X}$, the support of $X$, if $S_T(t|x)$ satisfies the Equation (\ref{const:2}) for all $t\in \mathcal{T}\subseteq \mathbb{R}^+$, then $S_T(t|x)=S_T^*(t|x)$  for $t\in \mathcal{T}$ almost surely. 
\end{proposition}

 %In the following section, we propose an alternative approach to solving Equation (\ref{const:2}) by constructing weighted self-consistent equations that can leverage self-consistent equations from other covariates to solve the self-consistent equation for each individual covariate.

\section{Reformulation of Self-Consistent Equations for Conditional Survival Functions} 

In subsection \ref{subsection:weightedsce}, we introduce a class of weighted self-consistent equations that can ensure Equation (\ref{const:2}) to hold under certain conditions. This approach generalizes the problem of solving self-consistent equations by embedding it within a broader class of equations. In subsection \ref{subsection:minmax}, we reformulate this broader class of problems as a min-max optimization, translating the task of solving infinitely many weighted self-consistent equations into an optimization framework.

\subsection{Weighted self-consistent equations}\label{subsection:weightedsce}

    %We first introduce an evaluation metric to quantify the extent to which $S(t|\cdot)$ satisfies the self-consistent equation at the population level for a fixed time $t$. Specifically, we 
    Define $D^P(t, S)$ as the square of the expected difference between the left and right sides of Equation (\ref{const:2}), specifically, 
        \[
        D^P(t,S)=\left\{E_X\left( S_T(t|X)-\left[ S_{T}^*(t|X)S_{C}^*(t|X)+E_C\left\{
        \frac{S_T^*(C|X)}{S_{T}(C|X)} I(C\leq t) |X\right\}S_T(t|X)\right]\right)\right\}^2,
        \]
    where the squared term ensures convexity and achieves a minimum value of zero when $S_T(t|X)$ satisfies equation (\ref{const:2}). On the other hand, deviations of $D^P(t,S)$ from zero indicate that for some $x \in \mathcal{X}$, $S_T(t|x)$ does not satisfy the self-consistent Equation (\ref{const:2}).
    
    We then extend $D^P(t, S)$ to $D^I(t, S, \phi)$, by incorporating a weighting function $\phi(X)$:
        \begin{equation}\label{const:1}
            \begin{split}
            D^I(t,S,\phi)=\left[E_X\left\{\left( S_T(t|X)-\left[ S_{T}^*(t|X)S_{C}^*(t|X)+E_C\left\{
        \frac{S_T^*(C|X)}{S_{T}(C|X)} I(C\leq t) |X\right\}S_T(t|X)\right]\right)\phi(X)\right\}\right]^2.
        \end{split}
        \end{equation}
    %Unlike $D^P(t,S)$, which provides a population-level assessment, $D^I(t,S,\phi)$ 
    By choosing different weight functions, it allows us to quantify how well $S_T(t|X)$ satisfies the self-consistent equation for a specific covariate value $x$. For example, we can set $\phi(X) = 1$ for $X = x$ and $0$ otherwise. In general, we can show that, if $D^I(t,S,\phi) = 0$ for $\phi \in \Phi_B$, where $\Phi_B \equiv \{\text{all non-negative, bounded functions with} \,\, \phi(\cdot) \le B\}$ for some $B > 0$, then $S(t|x)$ satisfies Equation (\ref{const:2}) for almost every $x \in \mathcal{X}$ at given time $t$.

        %$\phi \in \Phi_B := \{ \phi: \mathcal{X} \to \mathbb{R} \mid \phi(\cdot) \in [0, B] \}$     
    
    \begin{proposition}\label{prop:equiproblem} For a given \( t \in \mathbb{R}^+ \), \( D^I(t, S, \phi) = 0 \) for all \( \phi \in \Phi_B \) if and only if \( S(t|x) \) satisfies Equation (\ref{const:2}) for almost every \( x \in \mathcal{X} \).
    \end{proposition}
    
    Therefore, for a given $t$, minimizing Equation (\ref{const:1}) for all $\phi \in \Phi_B$ is equivalent to solving the self-consistent equations in Equation (\ref{const:2}) for all $x \in \mathcal{X}$.

\subsection{Min-max optimization for solving the class of weighted self-consistent equations}\label{subsection:minmax}

    %Proposition \ref{prop:equiproblem} demonstrates that, for a given time point $t$, solving the self-consistent equations for all $x \in \mathcal{X}$ is equivalent to finding a solution $S$ such that $D^I(t, S, \phi) = 0$ for every $\phi \in \Phi_B$. Furthermore, if $D^I(t, S, \phi) = 0$ holds for all $t \in \mathcal{T}$ and $\phi \in \Phi_B$, then, by Proposition \ref{uniquesol}, the solution satisfies $S_T(t|x) = S_T^*(t|x)$ for every $t \in \mathcal{T}$ and $x \in \mathcal{X}$.

    Instead of solving the weighted self-consistent equation for each $\phi \in \Phi_B$, which is generally infeasible,
    we formulate the problem under a min-max optimization framework. To achieve this, we introduce a loss function $C(S, \phi)$ defined as follows:
    \begin{equation} 
     C(S,\phi) = E_{V}[D^I(V,S,\phi)],
    \end{equation}
    where $V$ is a random variable following some distribution whose support matches that of observed event times, for example, $V$ can follow the empirical distribution of $\tilde{t}_i$, $i=1,\cdots,N$.  It can be seen that $C(S, \phi) = 0$ implies $D^I(t, S, \phi) = 0$ almost surely for $t \in \mathcal{V}$, where $\mathcal{V}$ is the support of $V$.%By verifying whether $D^I(V, S, \phi) = 0$, we can determine if the current $S$ satisfies the weighted self-consistent equations at time $V$. As an example, $V$ can be chosen to follow the distribution of the observed survival times $\tilde{T}$.

    % The distribution of $V$ can either match that of $\tilde{T}$ or be uniform, $V \sim U[0, \max(\{\tilde{T}_i\}_{i=1}^N)]$.

    %, in which case it uniformly spreads over the interval from 0 to the maximum observed time.

    %We can interpret $C(S, \phi)$ as the total deviation of the current $S$ from solving the self-consistent equations, weighted by $\phi$, across the range of $V$. Recall that $D^I(V, S, \phi)$ quantifies the deviation of $S$ from solving the self-consistent equation at a specific value of $V$. 
        
   % Our goal is to find an $S$ that ensures $C(S, \phi) = 0$ for all $\phi \in \Phi_B$. Rather than optimizing $S$ separately for each $\phi \in \Phi_B$, we identify the $\phi_B^* \in \Phi_B$ that maximizes $C(S, \phi)$, satisfying $C(S, \phi) \leq C(S, \phi_B^*)$. This implies that if $C(S, \phi_B^*) = 0$, then $C(S, \phi) = 0$ for every $\phi \in \Phi_B$. 
   This motivates us to consider the following min-max optimization:
    \begin{equation}\label{min-maxgame-population} \begin{split} \underset{S \in \mathcal{S}}{\min} \ \underset{\phi \in \Phi_B}{\max} \ C(S, \phi), \end{split} \end{equation}
    where $\mathcal{S}$ represents the class of conditional survival functions $S_T(t|x)$. We show that the proposed min-max optimization of $C(S,\phi)$ is equivalent to solving self-consistent equations (\ref{const:2}) for every $t\in \mathcal{V}$ and $x\in \mathcal{X}$.
    
    \begin{theorem}\label{thm:1} 
    The minimum of $C(S, \phi)$ exists and is equal to $0$. Furthermore, $\underset{\phi \in \Phi_B}{\max} \ C(S, \phi) = 0$ if and only if $ S_T(t|x)=S_T^*(t|x)$ for $t \in \mathcal{V}$ and $x \in \mathcal{X}$ almost surely.
    \end{theorem}

\begin{remark}
The motivation behind this strategy is analogous to the min-max optimization framework commonly used in Generative Adversarial Networks (GANs) \citep{fellow2014gan}. In GANs, two neural networks, a generator and a discriminator, are trained in opposition: the generator produces synthetic samples that mimic the true data distribution, while the discriminator assesses the authenticity of these samples. Usually, loss functions like the Jensen-Shannon divergence or the Wasserstein distance are used to measure the distance between the true and generated sample distributions. Similarly, in our approach, we treat $\phi$ as a discriminator and identify the function $\phi$ that maximizes $C(S, \phi)$ for a given $S$, thereby resulting in the maximal deviation of the current $S$ from satisfying the self-consistent equations. We then determine the $S$ that minimizes this $C(S, \phi)$, thus solving the self-consistent equations under the challenging 
$\phi$ identified in the previous step. Unlike the loss functions in GANs, which focus distance between distributions, the proposed loss function here is designed to solve the self-consistent equations. While $\phi$ plays a similar role of a GAN discriminator, it is more accurately understood as a weight function for the self-consistent equations.
\end{remark}

Lastly, we replace $C(S,\phi)$ with its empirical estimator, denoted as $C_{M,N}(S,\phi)$, that is, 
\begin{equation}\label{empiricalloss}
    \begin{split}
        C_{M,N}(S, \phi) = \frac{1}{M} \sum_{m=1}^M 
        \Bigg[ & \frac{1}{N} \sum_{i=1}^N S_T(V_m | x_i) \phi(x_i) \\
        & - \frac{1}{N} \sum_{i=1}^N \left\{ I(\tilde{t}_i > V_m) 
        + \frac{I(\delta_i = 0)}{S_T(\tilde{t}_i | x_i)} I(\tilde{t}_i \leq V_m) S_T(V_m | x_i) \right\} \phi(x_i)
        \Bigg]^2
        \end{split}
\end{equation}
where $V_1,\cdots,V_M$ are independent samples from the distribution of $V$. It is easy to verify that $C_{M,N}(S,\phi)$ converges to $C(S,\phi)$ as $M,N\to \infty$. Our proposed estimator of the conditional survival function is the solution to the following min-max optimization: 

%In conclusion, we construct the loss function $C(S, \phi)$, where the min-max solution with respect to $S$ corresponds to the true conditional survival function. By replacing $C(S, \phi)$ with $C_{M,N}(S, \phi)$, we propose the solution to min-max optimization of $C_{M,N}(S, \phi)$ in Equation (\ref{min-maxgame-empirical}) as our estimator for the conditional survival function:
\begin{equation}\label{min-maxgame-empirical}
    \underset{S \in \mathcal{S}}{\min} \ \underset{\phi \in \Phi_B}{\max} \ C_{M,N}(S, \phi).
\end{equation}

\section{\underline{S}elf-\underline{C}onsistent \underline{E}quation-Guided \underline{Ne}ural Net (SCENE)}

    The min-max optimization framework in Equation (\ref{min-maxgame-empirical}) provides a general methodology for estimating the conditional survival function. This framework supports a wide range of modeling approaches for the search space $\mathcal{S}$, ranging from parametric to fully non-parametric. %In the fully non-parametric setting, the framework provides an estimator for all conditional survival probabilities, $S(\tilde{T}_j|X_i)$, for $i, j = 1, \dots, N$. Alternatively, in a semi-parametric setting, the function space $\mathcal{S}$ can be restricted to structured models, such as Cox model-based neural nets.
    In the fully nonparametric setting, we can model the class of conditional survival functions using monotonic neural networks (MNN), as they enforce monotonicity in the estimated conditional survival functions by constraining the network weights to be positive \citep{daniels2010mnn, zhang2018cdf}.  However, enforcing the positivity constraint in parameters of neural nets introduces challenges in maintaining training stability, and identifying the optimal MNN architecture remains a nontrivial task. Here, we propose constructing the class of conditional survival functions using an empirical cumulative distribution function derived from a conditional distribution generator. This approach not only allows us to sufficiently approximate the whole class of conditional survival functions but also transforms the min-max optimization task into a continuous optimization problem, enabling the parameters of the conditional distribution generator to be updated iteratively using stochastic gradient methods. 
    
    \subsection{Proposed SCENE estimation}\label{subsection:functionclass}

   %To approximate the class of survival functions, Monotonic Neural Networks (MNN) are a suitable candidate for $\mathcal{S}$, 
   
    We adopt a conditional distribution generator \citep{Mirza2014conditional} to approximate the class of conditional survival functions. Specifically, we generate multiple samples $T_{ki}$, for $k = 1, \dots, K$, for each $X_i$ using a conditional generator $G_{\omega}$, a neural network parameterized by $\omega$. The conditional distribution generator $G_{\omega}$ takes two inputs: an auxiliary variable $U_k \in \mathbb{R}^{p_u}$, where $p_u$ is the dimension of the auxiliary variable, and the covariate $X_i$. The auxiliary variable $U_k$ is sampled from a pre-selected distribution $\pi(\cdot)$, such as the multivariate uniform or multivariate normal distribution. The output $T_{ki} = G_{\omega}(U_k, X_i)$ represents the $k$th generated sample of the survival time, conditional on the covariate $X_i$. The structure of the conditional generator is illustrated in Fig. \ref{fig:congenerator}. Consequently, the conditional survival function, constructed from the conditional distribution generator, can be expressed as in Equation (\ref{expression:ecdf}):
    \begin{equation}\label{expression:ecdf}
        \begin{split}
            S(t|X_i) &= \frac{1}{K} \sum_{k=1}^{K} I(T_{ki} > t), \quad \text{where} \quad T_{ki} = G_{\omega}(U_k, X_i), \quad U_k \sim \pi(u).
        \end{split}
    \end{equation}

\begin{figure}[!htbp]
    \centering
    \includegraphics[width=\textwidth]{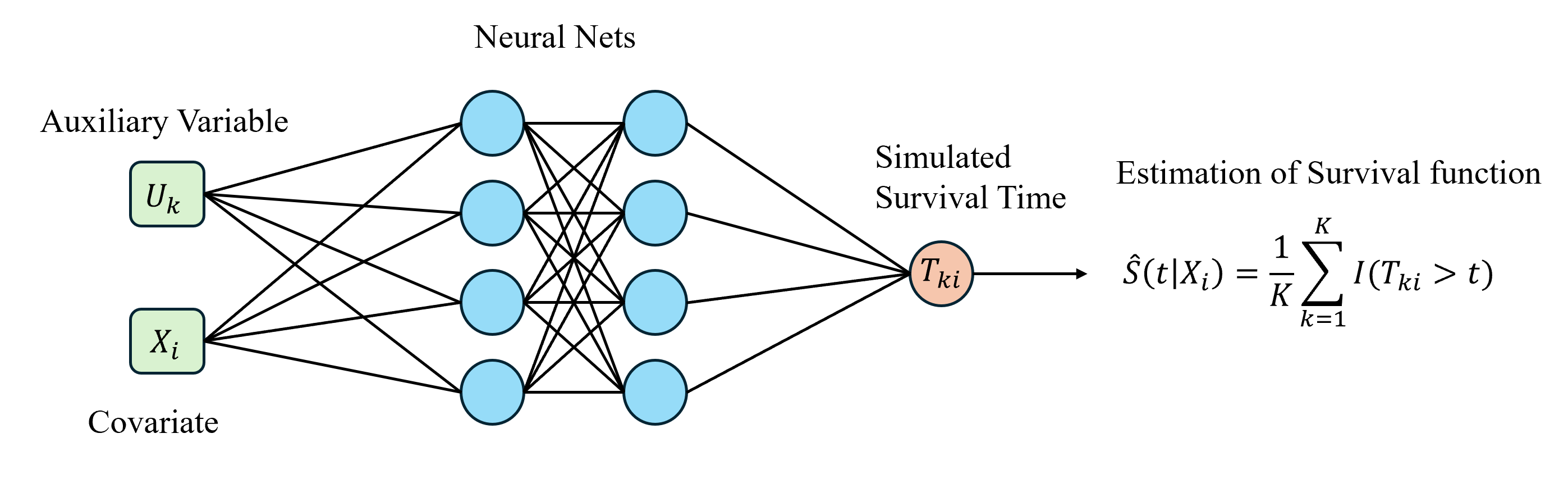}
    \caption{Illustration of the conditional generator used to approximate the survival function through generated samples.}
    \label{fig:congenerator}
\end{figure}

Similarly, we represent the weight functions $\phi \in \Phi_B$ by neural networks parameterized by $\zeta$ and denoted as $\phi_{\zeta} \in \Phi_\zeta^B$, where $\Phi_\zeta^B$ is a class of nonnegative neural network functions, bounded by $B$.
% \begin{remark} The conditional generator structure, introduced by \citet{Mirza2014conditional}, was originally designed to generate synthetic samples conditioned on label information.
% \end{remark}
Lastly, we use positive activation functions, such as \textit{exp} and \textit{sigmoid}, at the output of both $G_{\omega}$ and $\phi_{\zeta}$, since survival times and $\phi_{\zeta}$ must be non-negative. For example, in the case of a Multi-Layer Perceptron (MLP) with $l$ layers for $G_{\omega}$, let $\omega = \{\omega^w_1, \omega^b_1, \dots, \omega^w_L, \omega^b_L\}$ represent the set of weights ($\omega_{\cdot}^w$) and biases ($\omega_{\cdot}^b$), and let $\rho$ denote the activation function for all layers except the last. As an example for $\rho$, we can use \textit{ReLU} or \textit{tanh} functions. Then, the output of the generator denoted as $h_L^{\omega}$ can be expressed as:
\[
\begin{split}
    h_1^{\omega} &= \omega^w_1 (u,x) + \omega^b_1, \\
    h_l^{\omega} &= \omega^w_l \rho(h_{l-1}^{\omega}) + \omega^b_l, \quad l=2, \dots, L-1, \\
    h_L^{\omega} &= \omega^w_l h_{L-1}^{\omega} + \omega^b_L,
\end{split}
\]
where $l$ denotes the number of hidden layers, $\omega^w_1 \in \mathbb{R}^{m_1 \times (p_u+p)}, \omega^w_l \in \mathbb{R}^{m_l \times m_{l-1}}$ for $l=2, \dots, L-1$, and $\omega^w_L \in \mathbb{R}^{1 \times m_{L-1}}$. Additionally, $\omega^b_l \in \mathbb{R}^{m_l}$ for $l=1, \dots, L-1$ and $\omega^b_L \in \mathbb{R}$. Similarly, let the output of $\phi_{\zeta}$ be denoted as $h_{L}^{\zeta}$. We use the following activation functions to ensure the appropriate support for both $h_L^{\omega}$ and $h_L^{\zeta}$:
\[
\begin{split}
    G_{\omega}(u,x) &= \exp(h_L^{\omega})\in [0, \infty), \\
    \phi_{\zeta}(x) &= \frac{1}{1+\exp(-h_L^{\zeta})} \in [0, 1].
\end{split}
\]

In conclusion, the function classes we consider for Equation (\ref{min-maxgame-empirical}) are defined as 

\begin{equation}\label{functionclass}
\begin{split}
    \mathcal{S}_{\omega} &= \left\{ S(\cdot|x) : S(\cdot|x) = \frac{1}{K} \sum_{k=1}^{K} I(T_{k}(x) > \cdot), \; T_{k}(x) = G_{\omega}(U_k, x) \right\}\\
    \Phi_{\zeta}^B &= \left\{\phi: \phi(x)=B\phi_{\zeta}(x) \right\}
    \end{split}
\end{equation}
where $\mathcal{S}_{\omega}$ represents the class of conditional survival functions derived from the generator $G_{\omega}$ and $\Phi_{\zeta}^B$ represents the class of bounded non-negative function that can be expressed with neural nets. The proposed SCENE estimation is then given by the following min-max optimization:
\begin{equation}\label{min-maxgame-empirical-NN}
\begin{split}
        \underset{S\in \mathcal{S}_{w}}{\min} &\ \underset{\phi_{\zeta}\in\Phi^B_{\zeta}}{\max} \  C_{M,N}.
\end{split}
\end{equation}

For the above min-max optimization, the training procedure follows the standard iterative stochastic gradient descent update for the conditional generator and the weight function $\phi_{\zeta}$. The training algorithm is outlined in Algorithm \ref{alg:scene}.

\begin{algorithm}[!htbp]
%\SetAlgoLined
\begin{algorithmic}
 \STATE \textbf{Input}: total iteration number $H$, iteration for starting the variable selection $H_{VS}<H$, initialize the DNN weights $\omega^{(0)}$ and $\zeta^{(0)}$. Data $\mathcal{D} = \{\tilde{t}_i, \delta_i, x_i\}_{i=1}^N$, learning rate $\epsilon$; 
 
 \FOR{$h=1,2,\dots,H$}
        \STATE{ $\bullet$ Draw mini-batch with size $n$, denoted the by $\{(\tilde{t}_1,\delta_1,x_1),\dots,(\tilde{t}_n,\delta_n,x_n)$\}} \;
        \STATE{ $\bullet$ Draw $m$ time points, denoted by $t_1, \dots, t_m$, from $\{\tilde{t}_1, \dots, \tilde{t}_N\}$ through sampling without replacement} \;
  \FOR{$i=1,2,\dots,n$}
   \STATE{ $\bullet$ Draw $K$ auxiliary variables, denoted the by $U_1,\dots,U_K$} \;

   \STATE{ $\bullet$ Calculate survival function value for $t=\tilde{t}_1,\dots,\tilde{t}_n,t_1,\dots,t_m$} \;

    \[     
    S^{(h)}(t|x_i)=\frac{1}{K}\sum_{k=1}^K I(G_{\omega^{(h)}}(U_k,x_i)\geq t)
    \]  

\ENDFOR

\STATE{$\bullet$ Calculate loss} \;
\[
\begin{split}
   \tilde{C}(S^{(h)},\phi^{(h)})&=\frac{1}{m}\sum_{j=1}^m \{\tilde{L}(t_j,S^{(h)},\phi^{(h)})- \tilde{R}(t_j,S^{(h)},\phi^{(h)})\}^2 \\
     \tilde{L}(t,S^{(h)},\phi) &= \frac{1}{n}\sum_{i=1}^n S^{(h)}(t|x_i)\phi^{(h)}(x_i)\\
     \tilde{R}(t,S^{(h)},\phi)&=\frac{1}{n}\sum_{i=1}^n \{I(\tilde{t}_i>t)+\frac{I(\delta_i=0,\tilde{T}_i\leq t)}{S^{(h)}(\tilde{t}_i|x_i)}S^{(h)}(t|x_i)\}\phi^{(h)}(x_i)
\end{split}
\]

   \STATE{ $\bullet$ Update $S$: Take the gradient with respect $\omega$ on $\tilde{L}$ only, denoted as $\nabla_{\omega}^L$ } \;

   \[
   \omega^{(l+1)}=\omega^{(h)}-\epsilon\nabla_{\omega}^L \tilde{C}(S^{(h)},\phi^{(h)})
   \]

   \STATE{ $\bullet$ Update $\phi$(update $\zeta$) with updated $S^{(h+1)}$ with same proceduare at previous steps} \;

   \[
   \zeta^{(h+1)}=\zeta^{(h)}+\epsilon\nabla_{\zeta} \tilde{C}(S^{(h+1)},\phi^{(h)})
   \]
    \STATE{ {\bf if} variable selection is true and $h>H_{VS}$, {\bf then} calculate variable importance $\bgamma$ and threshold $\bar{\bgamma}_U := \frac{1}{p_u} \sum_{i=1}^{p_u} \bgamma_i$} by:

    \[\bgamma=|\omega_l^{w,(h+1)}|\times\dots\times|\omega_2^{w,(h+1)}|\times|\omega_1^{w,(h+1)}|\in \mathbb{R}^{1\times(p_u+p)}. 
    \]

    Then, find the index set $\mathcal{J} = \{ j : p_u + 1 \leq j \leq p_u + p, \, \bgamma_v \leq \bar{\bgamma}_U \}$ for variable selection. We set $\omega_1^w(i, \mathcal{J}) = 0$ for $i = 1, \dots, m_1$, where $\omega_1^w(i,j)$ refers to the $(i, j)$ component of $\omega_1^w$.
    
 \ENDFOR
 \caption{SCENE }\label{alg:scene} 
\end{algorithmic}
\end{algorithm}

% Let $\hat{S}^{\text{SCENE}}$ denote the solution of the game in Equation \ref{min-maxgame-empirical-NN}, as $\hat{S}^{\text{SCENE}}=\arg  \underset{S\in \mathcal{S}_{w}}{\min} \ \underset{\phi\in\Phi^B_{\zeta}}{\max} \  C_{M,N}(S,\phi)$ This solution, approximated using neural networks and finite sample sizes, is consistent with the true survival function as the sample size $N$, evaluation time points $M$ increas, as described in below Theorem \ref{thm:consistency}, whose proof is given at Appendix.

Define $\hat{S}^\text{SCENE} = \arg \underset{S \in \mathcal{S}_{\omega}}{\min} \ \underset{\phi \in \Phi^B_{\zeta}}{\max} \ C_{M,N}(S, \phi)$. Next, we establish the convergence rate of the SCENE estimator to the true conditional survival function $S_T^*(t|x)$, as the sample size $N \to \infty$ and the number of evaluation time points $M \to \infty$. To establish the results, we need the following conditions.

\begin{assumption}[Lipshitz and Curvature]\label{asm:lip_curvarture}

  For $C(S)=\underset{\phi\in \Phi_B}{\max} C(S,\phi)$, there exists some constants $C_{\phi},C_S,c_{1,C},c_{2,C},c_{1,D},c_{2,D}$ for any $S_1,S_2,S \in \mathcal{S}$ and $\phi_1,\phi_2,\phi\in\Phi_B$, the following inequalities
hold:
    \[
    \begin{split}
    |l(t,x,S_1,\phi_1)-l(t,x,S_2,\phi_2)|&\leq C_{\phi}|\phi_1(x)-\phi_2(x)|+C_{S}|S_1(t,x)-S_2(t,x)|\\
        c_{1,C}\|S(t,x)-S^*(t,x)\|^2_{L_2(t, x)|_{\mathcal{V}}}&\leq C(S)-C(S^*)\leq c_{2,C}\|S(t,x)-S^*(t,x)\|^2_{L_2(t, x)|_{\mathcal{V}}} \\
    c_{1,D}\|S(t,x)-S^*(t,x)\|^2_{L_2(x)}&\leq D^I(t,S,\phi)-D^I(t,S^*,\phi) \leq c_{2,D}\|S(t,x)-S^*(t,x)\|^2_{L_2(x)}
        \end{split}
    \]
where let $S(t|x)$ denote $S(t,x)$ for readability.
\end{assumption}

\begin{assumption}[Assumption 2 from \citep{Farrell_2021}]\label{asm:sobolev} For $\phi_{\hat{S}} = \arg \underset{\phi \in \Phi_B}{\max} C_{M,N}(\hat{S}, \phi)$, $S^*,\phi_{\hat{S}}$ lie in the Sobolev ball $\mathcal{W}^{\beta,\infty}([-1,1]^{p+p_u})$ and $\mathcal{W}^{\beta,\infty}([-1,1]^p)$, with smoothness $\beta\in\mathbb{N}_+$, 

\[
\begin{split}
    S^*(\cdot|x)\in \mathcal{W}^{\beta,\infty}([-1,1]^{p+p_u})&:=\big\{S:\max_{\alpha,|\alpha|\leq \beta}\esssup_{x\in[-1,1]^p}|D^{\alpha} S^*(\cdot|x)|\leq 1 \big\}\\
\phi^*(x)\in \mathcal{W}^{\beta,\infty}([-1,1]^p)&:=\big\{S:\max_{\alpha,|\alpha|\leq \beta}\esssup_{x\in[-1,1]^p}|D^{\alpha} \phi^*(x)|\leq 1 \big\}
\end{split}
\]
where $\alpha=(\alpha_1,\dots,\alpha_p)$, $|\alpha|=\alpha_1+\dots+\alpha_p$ and $D^{\alpha}f$ is the weak derivative.
\end{assumption}

Assumption \ref{asm:lip_curvarture} specifies curvature conditions for the loss function used in our framework. These conditions are natural and widely assumed in many contexts \citep{farrell2020DeepLearning,Farrell_2021}. Assumption \ref{asm:sobolev} imposes regularity properties on the functions to be approximated, which are widely used in the literature \citep{DeVore1989OptimalNA,yarotsky2017,Farrell_2021}. In our case, these conditions apply to the true survival functions and weight functions.

\begin{theorem}\label{thm:consistency} Suppose Assumptions \ref{asm:lip_curvarture} and \ref{asm:sobolev} hold. And assume $W_{S}, U_S \asymp N^{(p + p_u) / (2\beta + p + p_u)}$, $W_{\phi}, U_{\phi} \asymp N^{p / (2\beta + p)}$, and depth $L_S, L_{\phi} \asymp \log N$, where $W_{\cdot}, U_{\cdot}, L_{\cdot}$ denote the number of parameters, hidden units, and layer size for the function class of $\cdot$, parameterized in a multilayer perceptron and $\beta\in\mathbb{N}_+$ for smoothness constant of the true conditional survival function. Then, if $M = O(N)$, with probability at least $1 - \exp\left(-N^{\frac{1}{4}} M^{\frac{1}{4}}\right)$, the following result holds: 
\[
\|\hat{S}^{SCENE}-S^*\|^2_{{L_2(t,x)}|_{\mathcal{V}}}\leq C(K^{-1}N^{-\rho_1}+N^{-\rho_2}+M^{-1/2}+K^{-2}) 
\]
for some constant $C$ that does not depend on $N,M$ or $K$, and $\rho_1=\min\left(\frac{1}{4}, \frac{\beta}{2\beta + p + p_u} \right)$, $\rho_2 = \min\left(\frac{1}{4} + \frac{\beta}{2\beta + p + p_u}, \frac{2\beta}{2\beta + p + p_u}, \frac{\beta}{2\beta + p}\right)$, where $L_2(t, x)|_{\mathcal{V}}$ is the $L_2$ norm with respect to $t$ and $x$, and $\|f(t|x) - g(t|x)\|^2_{L_2(t, x)|_{\mathcal{V}}} = \int_{t\in \mathcal{V},x\in\mathcal{X}} \{f(t|x) - g(t|x)\}^2 \, dx \, dt $. 
\end{theorem}

%In summary, we propose the \textit{Self-Consistent Equation-guided Neural Net} (SCENE) for estimating the conditional survival function using generative neural nets trained through min-max optimization. 

%\begin{remark} When solving the self-consistent equation for the KM estimator, we iteratively update the value of $S^{(l)}(t)$, the survival function value at the $l^{\text{th}}$ iteration, for $l = 0, 1, \dots$, until convergence by
%\[
%\begin{split}
%    S^{(l+1)}(t)&=\frac{1}{N}\sum_{i=1}^N \left\{I(\tilde{T}_i>t) + \frac{(1-\delta_i)S^{(l)}(t)}{S^{(l)}(\tilde{T}_i)}S^{(l)}(t)\right\}.
%\end{split}
%\]

%Similarly, when calculating the gradient with respect to $\omega$ for SCENE, we take the gradient with respect to the left-hand side of Equation (\ref{empiricalloss}). 
%\end{remark}

\subsection{Variable importance }\label{subsection:variableselection}

The SCENE method can be enhanced by incorporating important predictor identification through network structure selection \citep{SunSLiang2021}. Similar to random survival forests, which report Variable Importance (VIMP), SCENE can also provide analogous values to rank the importance of variables. However, unlike random survival forests, which calculate these values only after training -- rendering them unavailable during the training process -- SCENE can efficiently compute these values in real-time, enabling their integration into model optimization.

First, we consider the network structure $\bgamma$ as in \cite{SunSLiang2021}:
\[
\bgamma=|\omega_L^w|\times\dots\times|\omega_2^w|\times|\omega_1^w|\in \mathbb{R}^{1\times(p_u+p)},
\]
where $|\omega_{\cdot}^w|$ represents the matrix of absolute values, with each element being the absolute value of the corresponding element in $\omega_{\cdot}^w$, and $\times$ represents the matrix multiplication. We can interpret $\bgamma$ as the total weights assigned to each variable, serving as a measure of the variable's importance to the network. A larger value of $\bgamma_j$ indicates that the $j$th variable is more important to the network. Conversely, if $\bgamma_j = 0$, the $j$th variable is not important and has no effect on the neural network, thus network structure can be used for variable selection. To train sparse neural network model and perform variable selection, \cite{SunSLiang2021} introduced a Gaussian mixture prior on the weights, defined as $\pi(\omega) \sim \lambda N(0, \sigma_{1}^2) + (1 - \lambda) N(0, \sigma_{0}^2)$, where $\lambda$, $\sigma_0$, and $\sigma_1$ are hyperparameters, with $\sigma_0$ being relatively small compared to $\sigma_1$. Then by solving the inequality $P(\omega\sim N(0,\sigma_0^2))\leq P(\omega \sim N(0,\sigma_1^2))$, they specified the threshold to construct $\tilde{\omega}$ as shown below:
% To perform variable selection aimed at a sparse neural network model that ensures consistent training, \citep{SunSLiang2021} {\color{green} originally assume} {\color{red} is this what we did also?} a Gaussian mixture prior on weights, defined as $\pi(\omega) \sim \lambda N(0, \sigma_{1}^2) + (1 - \lambda) N(0, \sigma_{0}^2)$ with hyperparameters $\lambda$, $\sigma_0$, and $\sigma_1$, where $\sigma_0$ is relatively small compared to $\sigma_1$. Then by solving the inequality $P(\omega\sim N(0,\sigma_0^2))\leq P(\omega \sim N(0,\sigma_1^2))$, we can specify the threshold to construct $\tilde{\omega}$ as shown below:
\[
\tilde{\omega} = \omega I \left\{
 |\omega| \geq \sqrt{\log\left(\frac{1 - \lambda}{\lambda} \frac{\sigma_{1}}{\sigma_{0}}\right) \frac{2 \sigma_{0}^2 \sigma_{1}^2}{\sigma_{1}^2 - \sigma_{0}^2}}\right\},
\]
which represents the variable-inclusion weights of the neural network.
Using $\tilde{\omega}$, they computed $\bgamma$, which can be used for variable selection. 

Note that $\lambda$, $\sigma_0$, and $\sigma_1$ are hyperparameters that impact the determination of the network structure. However, SCENE does not require such hyperparameters, as it assumes that $\bgamma_{1:p_u} = (\bgamma_1, \bgamma_2, \dots, \bgamma_{p_u})$, corresponding to the auxiliary variables $U_1,\dots,U_{p_u}$, are effective in generating survival times conditional on covariates. Thus, the importance of the $j$th covariate, $j=1,\dots,p$, can be evaluated by comparing $\bgamma_{p_u+j}$ to the average $\frac{1}{p_u}\sum_{i=1}^{p_u}\bgamma_i$. Here, 
 the average $\frac{1}{p_u}\sum_{i=1}^{p_u}\bgamma_i$ can be regarded as the baseline measure of importance for constructing the conditional generator for survival times without using covariate information. Specifically, if $\bgamma_{p_u+j} > \frac{1}{p_u}\sum_{i=1}^{p_u}\bgamma_i$, then the $j$th covariate is important for constructing the conditional generator for survival times. Consequently, including only those variables with importance values exceeding the average importance value of the auxiliary variables during the SCENE training process can enhance the performance of training through variable selection, especially when there is a large number of predictors. Such a variable selection procedure can be easily incorporated into the training process as follows:
\begin{enumerate} \item After a burn-in period, calculate $\bgamma \in \mathbb{R}^{p_u + p}$ and compute the threshold $\bar{\bgamma}_U := \frac{1}{p_u} \sum_{i=1}^{p_u} \bgamma_i$. \item Identify the index set $\mathcal{J} = \{ j : p_u + 1 \leq j \leq p_u + p, \bgamma_v \leq \bar{\bgamma}_U \}$ and for all $i = 1, \dots, m_1$, set $\omega_1^w(i, \mathcal{J}) = 0$, where $\omega_1^w(i,j)$ refers to the $(i, j)$ component of $\omega_1^w$. \end{enumerate}

\section{Simulations}

In this section, we evaluate the performance of SCENE using simulated survival data across various settings. Specifically, we consider combinations of the following factors: (1) model, (2) covariate dimensionality, and (3) censoring rate.

We consider the proportional hazards (PH) model and the proportional ddds (PO) model. The PH model has the following conditional survival function:
\[
    S(t|X_i) = e^{-\lambda \exp\{f(X_i)\}t}, 
\]
where $f(x_i) = -(x_{i,1}^2 + x_{i,2}^2)/(2r^2)$ with $\lambda = \log(0.1)$ and $r = 0.7$. The PO model has the following conditional survival function:
\[
    S(t|X_i) = \frac{1}{1 + t\exp\{f(X_i)\}}, 
\]
where $f(x_i) = -(x_{i,1}^2 + x_{i,2}^2)/(2r^2)$ with $r = 0.5$.

For the dimension of covariates, we considered both low-dimensional and high-dimensional cases. The covariates $X_{ij}$ were independently sampled from a uniform distribution $U[-1, 1]$ for $i = 1, \dots, N$ and $j = 1, \dots, p$, with $p = 5$ for the low-dimensional case and $p = 100$ for the high-dimensional case. %A total of $N = 4000$ observations were used in each experiment.

Lastly, we considered moderate and high censoring rates, approximately $20\%$ and $50\%$, respectively. The censoring times were generated from a uniform distribution $U[0, \tau]$, where $\tau$ was chosen to control the censoring rate. 
%Smaller values of $\tau$ lead to higher censoring rates. 
The resulting average censoring rates for different values of $\tau$ are summarized in Table \ref{tab:cenratio}.

     % \item YP Model: 

     %    \[
     %    \begin{split}
     %              h(t|X_i)&=h_0(t)\frac{\lambda(X)\theta(X)}{\lambda(X)F_0(t)+\theta S_0(t)}\\
     %           S(t|X)&=[1+\frac{\lambda(X)}{\theta(X)}\frac{F_0(t)}{S_0(t)}]^{-\theta(X)}  
     %    \end{split}
     %    \]

\begin{table}[!htbp]
\caption{Average censoring ratios over 100 datasets for different $\tau$ values at $p=5$ and $p=100$}
\label{tab:cenratio}
\vspace{-0.2in}
\begin{center}
%\begin{adjustbox}{width=1.0\textwidth}
\begin{tabular}{cc c cc } \toprule
       \multicolumn{2}{c}{PH Model} & & \multicolumn{2}{c}{PO Model} 
        \\ \cline{1-2} \cline{4-5}    
 $\tau$=5 & $\tau$=19 &  &  $\tau$=5 & $\tau$=35
 \\ \midrule
$52.98\%$ & $20.49\%$ & & $53.10\%$ & $20.36\%$ \\
\bottomrule  
\end{tabular}
 %\end{adjustbox}
% \vspace{-0.2in}
\end{center}
\end{table}

    % Empirical Pointwise interval? Prediction variability? Quantify sampling variability
    % ECI?
        
    For each of the 8 scenarios, we generated 100 datasets of size $N = 4000$. For each dataset, we applied DeepSurv \citep{Katzman2018dsurv}, random survival forests \citep{strobl2007biasrsf}, and SCENE to estimate the conditional survival functions. The implementation details for SCENE are provided in the supplementary material.

    The performance of SCENE was evaluated from two perspectives: (1) bias and variability of the estimates and (2) distributional properties of the survival times generated by SCENE. To assess bias and variability, we selected four fixed test individuals, each with covariates  $x^{\text{test}}_{i}$ for $i = 1, \ldots, 4$. Each covariate of the first individual was randomly sampled from its support, i.e., $x^{\text{test}}_{1j} \sim \text{Unif}[-1, 1]$ for $j = 1, \ldots, p$. For the second, third, and fourth test individuals, their covariates were set as $x_{2j}^{\text{test}} = 0.25$,  $x_{3j}^{\text{test}} = 0.5$, $x_{4j}^{\text{test}} = 0.75$, $j=1,\dots,p$, such that the corresponding risk scores $f(x_i^{\text{test}})$, $i=2, 3, 4$, equal to the 9.8\%, 39.2\%, and 85.5\% quantiles of the risk scores $f(X)$, respectively. Then, for each time point $t \in [0, \tau]$, we computed $\hat{S}^{\cdot, l}(t | x^{\text{test}}_i)$, where $l = 1, \dots, 100$ denotes the estimate derived from the $l$th dataset for a given scenario, and $\cdot$ represents the specific estimator, such as DeepSurv, random survival forests, or SCENE. From these estimates, we derived the 5\% and 95\% quantiles of predicted survival probabilities, $\{\hat{S}^{\cdot, l}(t | x^{\text{test}}_i)\}_{l=1}^{100}$, referred to as the 90\% pointwise \textit{empirical bound} for the predicted conditional survival function given a test individual's covariates, which provides an empirical measure of prediction variability. Additionally, we assessed the bias of the conditional survival function estimates by comparing the average of $\{\hat{S}^{\cdot, l}(t | x^{\text{test}}_i)\}_{l=1}^{100}$ with the true survival function $S(t | x^{\text{test}}_i)$.

    We used a QQ-plot to evaluate whether the samples generated by SCENE share the same distributional properties as the true survival times. First, we defined evenly spaced quantiles $0 = q_0 < q_1 < \ldots < q_Q = 1$, where $q_1 = 0.01, q_2 = 0.02, \ldots, q_Q = 1$. Let $F_T|X$ and $F_{\hat{S}^l}|X$ denote the cumulative distribution functions of survival time given $X$ for the true survival times and those generated by SCENE, trained on the $l$th dataset for $l=1,\dots,100$, respectively. The QQ-plot compares these quantiles by scatter plotting $(F^{-1}_T(q_i|X), F^{-1}_{\hat{S}^l}(q_i|X))$ for $i = 0, \ldots, Q$, where $F^{-1}_{\hat{S}^l}(q_i|X)$ represents the $q_i$-quantile of the generated samples $\{G_{\hat{\omega}}(U_k, X)\}_{k=1}^K$. As with the bias and variability assessment, empirical bounds were constructed for the QQ-plot. This plot provides a way to evaluate SCENE’s ability to
    generate survival times, complementing its capacity to estimate survival functions. 
    %distinguishing it from methods that only provide survival function estimates.

     We present results for the high censoring rate case in the main paper, as it represents a more challenging scenario. Results for the moderate-censoring rate case are provided in the supplementary material.

    The evaluation results from these two perspectives across various scenarios are summarized in Figures \ref{fig:phc5n4000d5estimate} to \ref{fig:poc5n4000d100qq}. Figure \ref{fig:phc5n4000d5estimate} (PH model, low-dimensional), Figure \ref{fig:poc5n4000d5estimate} (PO model, low-dimensional), Figure \ref{fig:phc5n4000d100estimate} (PH model, high-dimensional), and Figure \ref{fig:poc5n4000d100estimate} (PO model, high-dimensional) show estimates of conditional survival functions from 100 datasets on for four test subjects for SCENE, Deepsurv and random survival forests. Figure \ref{fig:phc5n4000d5qq} (PH model, low-dimensional), Figure \ref{fig:poc5n4000d5qq} (PO model, low-dimensional), Figure \ref{fig:phc5n4000d100qq} (PH model, high-dimensional), and Figure \ref{fig:poc5n4000d100qq} (PO model, high-dimensional) present QQ plots comparing the true survival times to the generated survival times from SCENE.

    %{\color{red} clarify results from which of the 8 scenarios are reported here?} 
    % This design allowed for constructing empirical pointwise confidence bands using the estimators obtained for each method across the 100 datasets.

    % More in detail, the four fixed test subjects {\color{blue} were selected to consistently evaluate the performance of SCENE and other methods, denoted as $x^{test}_{i}$ for $i = 1, \ldots, 4$. }{\color{red} what do you mean by "test" subjects?} The first individual'scovariates were randomly sampled from their support, i.e., $x^{test}_{1j} \sim Unif[-1, 1]$ for $j = 1, \ldots, p$. For the second, third, and fourth individuals, the covariates were set to the 25\%, 50\%, and 85.5\% quantiles of the \textcolor{blue}{ risk score $f(\beta,x_i)$},{\color{red} what do you mean here? this can't be correct?} specifically $x^{test}_{2j} = 0.25$, $x^{test}_{3j} = 0.5$, and $x^{test}_{4j} = 0.75$ for $j = 1, \ldots, p$. {\color{red} looks like these covariate values are fixed?} 
    % {\color{red} need to first describe what you did and what were provided in these figures.} 

    In the low-dimensional case under the PH model, Figure \ref{fig:phc5n4000d5estimate} demonstrates that all methods effectively estimated the conditional survival functions for all test subjects, as their empirical bounds covered the true survival function and their average estimates closely aligned with the true values. Notably,  SCENE achieved the narrowest empirical bounds across all test subjects. 
        
\begin{figure}[!htbp]
    \centering
    \begin{subfigure}{.5\textwidth}
        \centering
        \includegraphics[width=\linewidth]{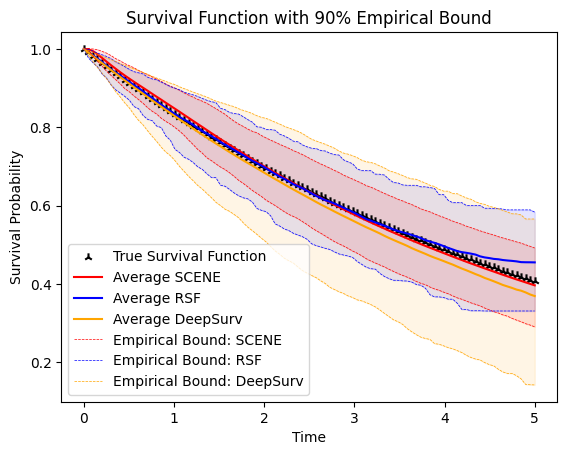}
        \subcaption*{(a) Subject 1: Random sampling}
    \end{subfigure}%
    \begin{subfigure}{.5\textwidth}
        \centering
        \includegraphics[width=\linewidth]{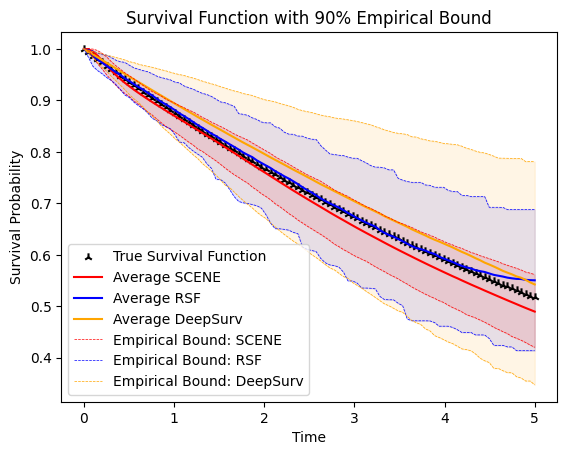}
        \subcaption*{(b) Subject 2: $9.8\%$ Quantile of risk score}
    \end{subfigure}
    \begin{subfigure}{.5\textwidth}
        \centering
        \includegraphics[width=\linewidth]{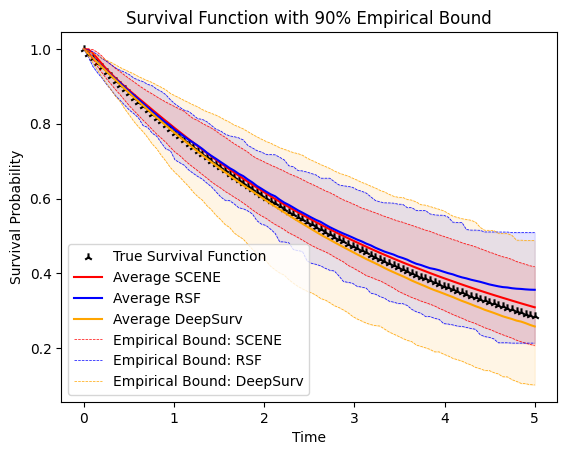}
        \subcaption*{(c) Subject 3: $39.2\%$ Quantile of risk score}
    \end{subfigure}%
    \begin{subfigure}{.5\textwidth}
        \centering
        \includegraphics[width=\linewidth]{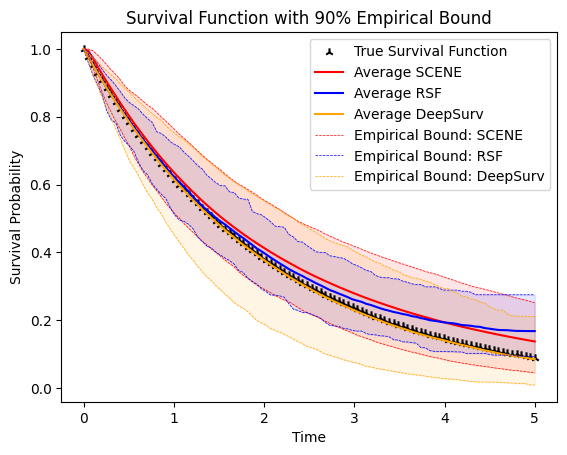}
        \subcaption*{(d) Subject 4: $85.5\%$ Quantile of risk score}
    \end{subfigure}
    \caption{Comparison of conditional survival function estimation for PH Model, $C=5$, $N=4000$, $d=5$: $(5\%, 95\%)$ empirical band for Test Subject 1 to Test Subject 4.}
    \label{fig:phc5n4000d5estimate}
\end{figure}

In the low-dimensional case under the PO model, Figure \ref{fig:poc5n4000d5estimate} illustrates that random survival forests exhibited noticeable bias near the end of the observed survival time for Test subject 1, while SCENE showed some bias for Test subject 3. Overall, SCENE achieved the narrowest empirical bounds for Test Subjects 1 through 4, consistent with its performance under the PH model.

\begin{figure}[!htbp]
    \centering
    \begin{subfigure}{.5\textwidth}
        \centering
        \includegraphics[width=\linewidth]{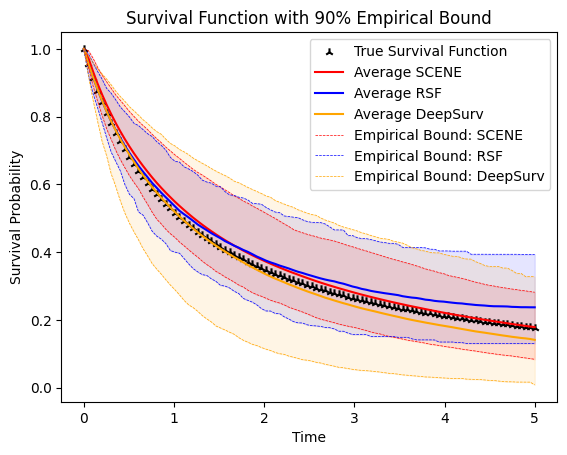}
        \caption*{(a) Subject 1: Random sampling}
    \end{subfigure}%
    \begin{subfigure}{.5\textwidth}
        \centering
        \includegraphics[width=\linewidth]{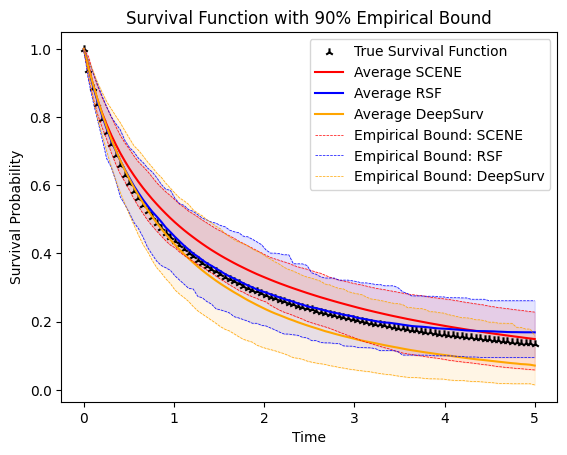}
        \caption*{(b) Subject 2: $9.8\%$ Quantile of risk score}
    \end{subfigure}
    \begin{subfigure}{.5\textwidth}
        \centering
        \includegraphics[width=\linewidth]{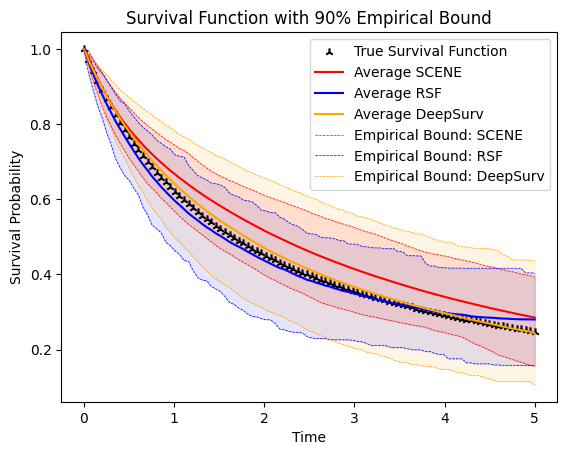}
        \caption*{(c) Subject 3: $39.2\%$ Quantile of risk score}
    \end{subfigure}%
    \begin{subfigure}{.5\textwidth}
        \centering
        \includegraphics[width=\linewidth]{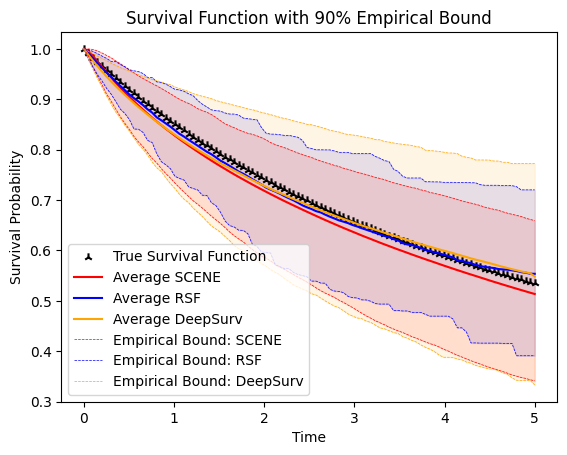}
        \caption*{(d) Subject 4: $85.5\%$ Quantile of risk score}
    \end{subfigure}
    \caption{Comparison of conditional survival function estimation for PO Model, C=5, $N=4000$, $d=5$ : $(5\%,95\%)$ empirical bound for Test Subject 1 to Subject 4. }
        \label{fig:poc5n4000d5estimate}
        \vspace{-8mm}
\end{figure}

\newpage

    As the next step, we considered a more complex scenario with high-dimensional covariates. The estimation results are shown in Figure \ref{fig:phc5n4000d100estimate} for the PH model and Figure \ref{fig:poc5n4000d100estimate} for the PO model. Among the methods evaluated, DeepSurv performed the worst, exhibiting severe bias and wide empirical bounds due to its lack of a variable selection mechanism, which limits its ability to handle high-dimensional data. Random survival forests showed narrower empirical bounds but exhibited consistent bias across all test subjects. In contrast, SCENE generally achieved narrow empirical bounds that included the true survival probability values, with moderate bias observed for Test Subjects 1 and 3.

\begin{figure}[!htbp]
    \centering
    \begin{subfigure}{.5\textwidth}
        \centering
        \includegraphics[width=\linewidth]{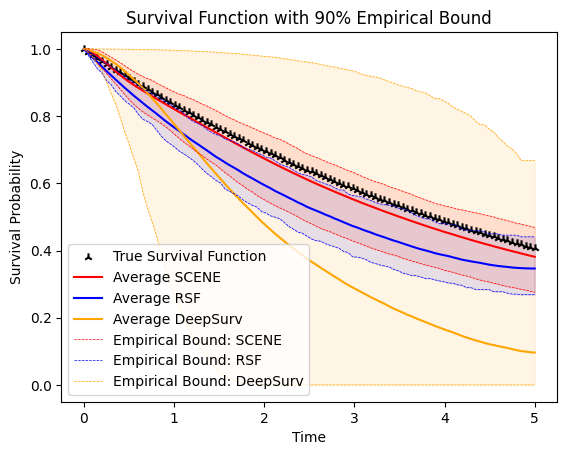}
        \caption*{(a) Subject 1: Random sampling}
    \end{subfigure}%
    \begin{subfigure}{.5\textwidth}
        \centering
        \includegraphics[width=\linewidth]{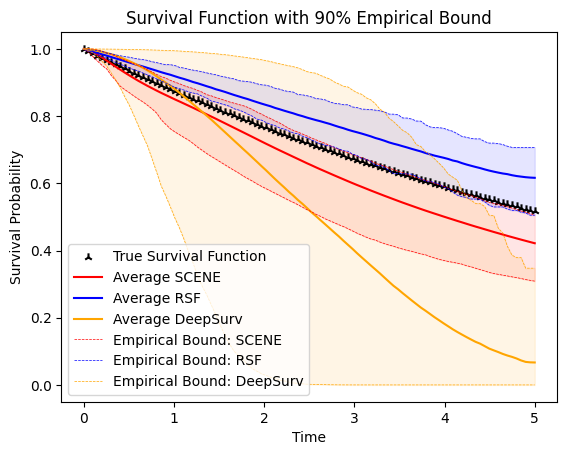}
        \caption*{(b) Subject 2: $9.8\%$ Quantile of risk score}
    \end{subfigure}
    \begin{subfigure}{.5\textwidth}
        \centering
        \includegraphics[width=\linewidth]{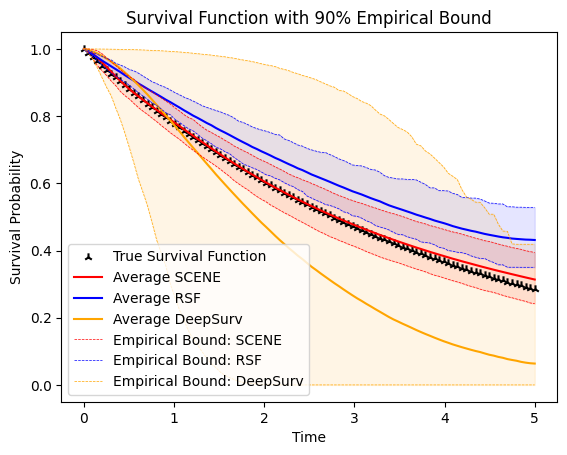}
        \caption*{(c) Subject 3: $39.2\%$ Quantile of risk score}
    \end{subfigure}%
    \begin{subfigure}{.5\textwidth}
        \centering
        \includegraphics[width=\linewidth]{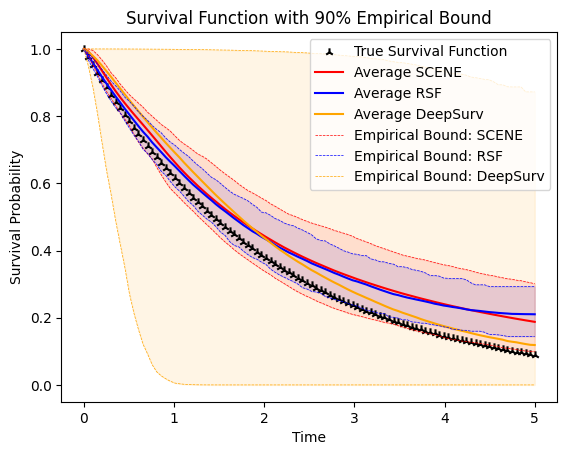}
        \caption*{(d) Subject 4: $85.5\%$ Quantile of risk score}
    \end{subfigure}
    \caption{Comparison of conditional survival function estimation for PH Model, C=5, $N=4000$, $d=100$ : $(5\%,95\%)$ empirical bound for Test Subject 1 to Subject 4.}
    \label{fig:phc5n4000d100estimate}
\end{figure}

\begin{figure}[!htbp]
    \centering
    \begin{subfigure}{.5\textwidth}
        \centering
        \includegraphics[width=\linewidth]{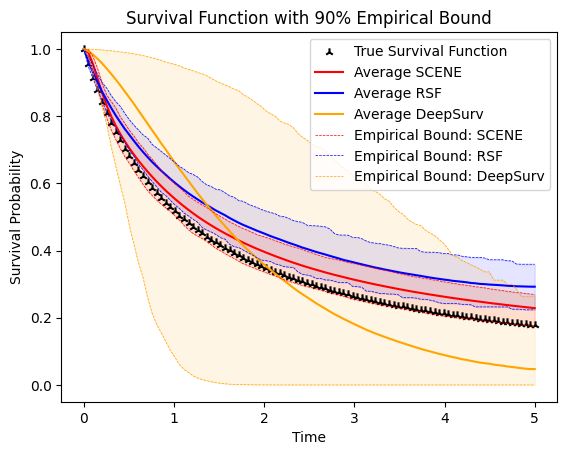}
        \caption{Subject 1: Random sampling}
    \end{subfigure}%
    \begin{subfigure}{.5\textwidth}
        \centering
        \includegraphics[width=\linewidth]{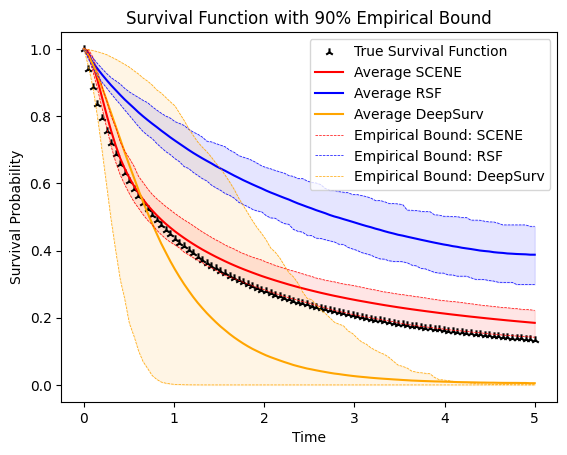}
        \caption{Subject 2: $9.8\%$ Quantile of risk score}
    \end{subfigure}
    \begin{subfigure}{.5\textwidth}
        \centering
        \includegraphics[width=\linewidth]{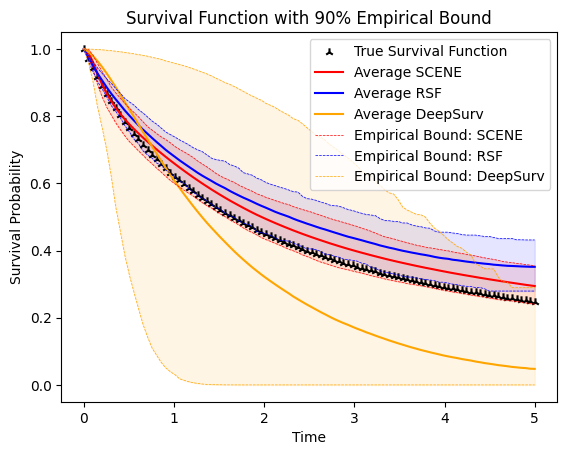}
        \caption{Subject 3: $39.2\%$ Quantile of risk score}
    \end{subfigure}%
    \begin{subfigure}{.5\textwidth}
        \centering
        \includegraphics[width=\linewidth]{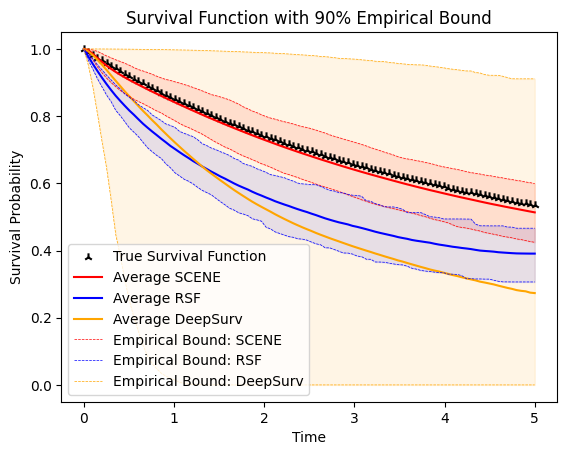}
        \caption{Subject 4: $85.5\%$ Quantile of risk score}
    \end{subfigure}
    \caption{Comparison of conditional survival function estimation for PO Model, C=5, $N=4000$, $d=100$ : $(5\%,95\%)$ empirical bound for Test Subject 1 to Subject 4.}
    \label{fig:poc5n4000d100estimate}
        \vspace{-8mm}
\end{figure}

\newpage
%%%%%%%%%%%%%%% QQ Plots

Additionally, Figures \ref{fig:phc5n4000d5qq}, \ref{fig:poc5n4000d5qq}, \ref{fig:phc5n4000d100qq}, and \ref{fig:poc5n4000d100qq} present QQ-plots demonstrating that the generated survival times closely aligned with the distributional properties of the true survival times across all scenarios.

\begin{figure}[!htbp]
    \centering
    \begin{subfigure}{.5\textwidth}
        \centering
        \includegraphics[width=\linewidth]{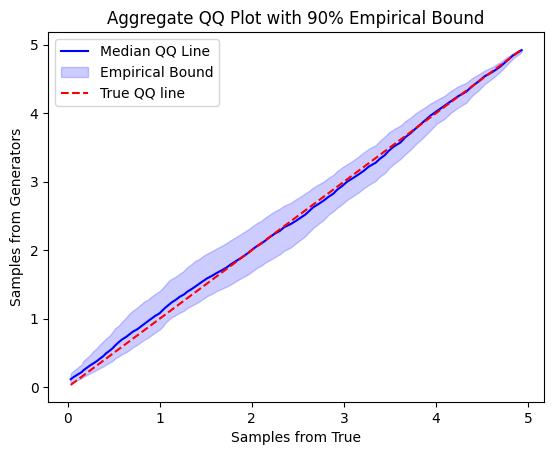}
        \caption{Subject 1: Random sampling}
    \end{subfigure}%
    \begin{subfigure}{.5\textwidth}
        \centering
        \includegraphics[width=\linewidth]{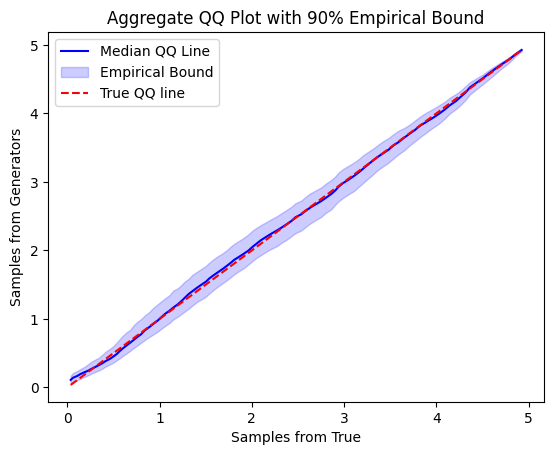}
        \caption{Subject 2: $9.8\%$ Quantile of risk score}
    \end{subfigure}
    \begin{subfigure}{.5\textwidth}
        \centering
        \includegraphics[width=\linewidth]{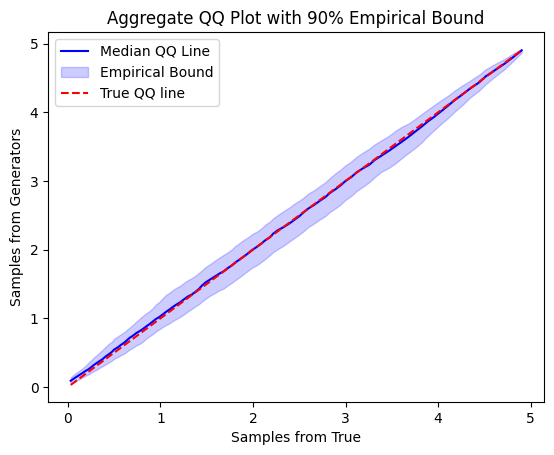}
        \caption{Subject 3: $39.2\%$ Quantile of risk score}
    \end{subfigure}%
    \begin{subfigure}{.5\textwidth}
        \centering
        \includegraphics[width=\linewidth]{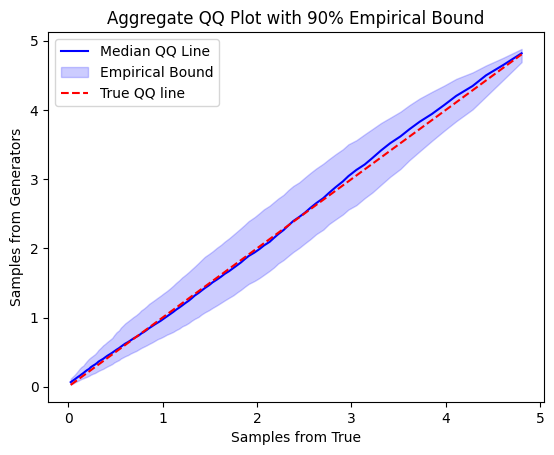}
        \caption{Subject 4: $85.5\%$ Quantile of risk score}
    \end{subfigure}
    \caption{QQ plot of true conditional samples ($x$-axis) and generated samples ($y$-axis) for PH Model, C=5, $N=4000$, $d=5$ : $(5\%,95\%)$ empirical bound for Test Subject 1 to Subject 4.}
    \label{fig:phc5n4000d5qq}
\end{figure}

\begin{figure}[!htbp]
    \centering
    \begin{subfigure}{.5\textwidth}
        \centering
        \includegraphics[width=\linewidth]{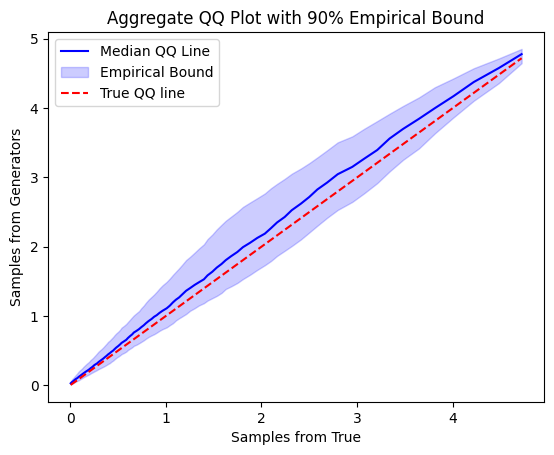}
        \caption{Subject 1: Random sampling}
    \end{subfigure}%
    \begin{subfigure}{.5\textwidth}
        \centering
        \includegraphics[width=\linewidth]{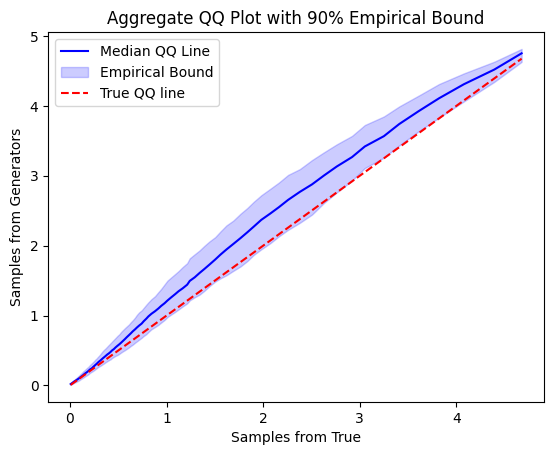}
        \caption{Subject 2: $9.8\%$ Quantile of risk score}
    \end{subfigure}
    \begin{subfigure}{.5\textwidth}
        \centering
        \includegraphics[width=\linewidth]{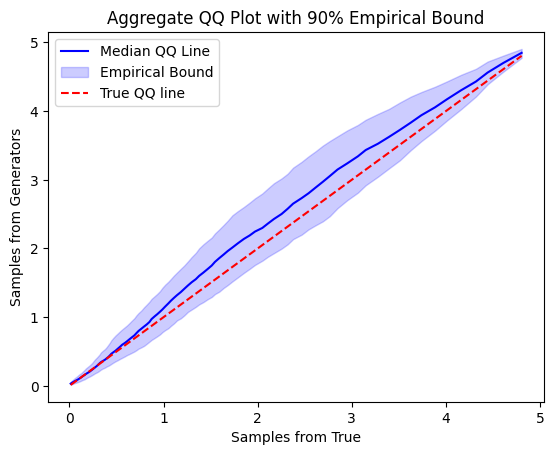}
        \caption{Subject 3: $39.2\%$ Quantile of risk score}
    \end{subfigure}%
    \begin{subfigure}{.5\textwidth}
        \centering
        \includegraphics[width=\linewidth]{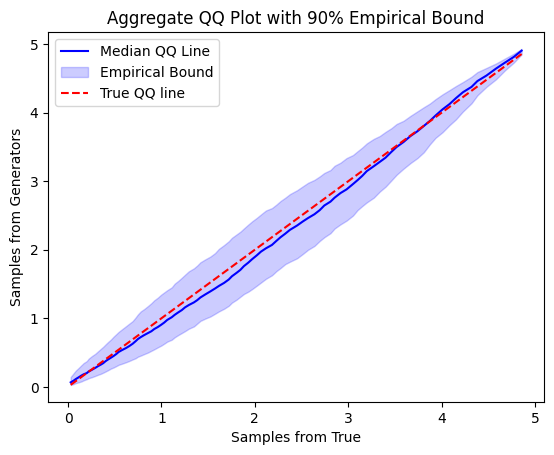}
        \caption{Subject 4: $85.5\%$ Quantile of risk score}
    \end{subfigure}
    \caption{QQ plot of true conditional samples ($x$-axis) and generated samples ($y$-axis) for PO Model, C=5, $N=4000$, $d=5$ : $(5\%,95\%)$ empirical bound Test Subject 1 to Subject 4.}
    \label{fig:poc5n4000d5qq}
\end{figure}

\begin{figure}[!htbp]
    \centering
    \begin{subfigure}{.5\textwidth}
        \centering
        \includegraphics[width=\linewidth]{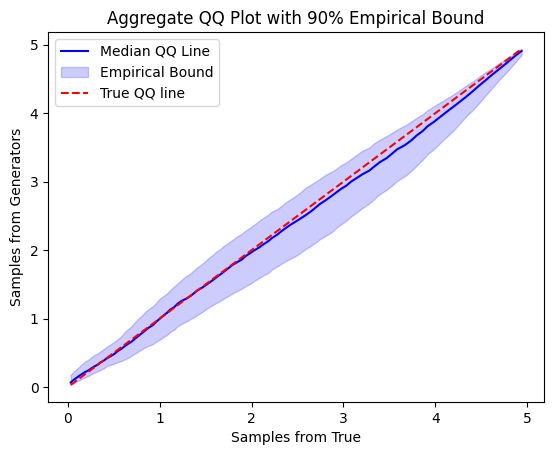}
        \caption*{(a) Subject 1: Random sampling}
    \end{subfigure}%
    \begin{subfigure}{.5\textwidth}
        \centering
        \includegraphics[width=\linewidth]{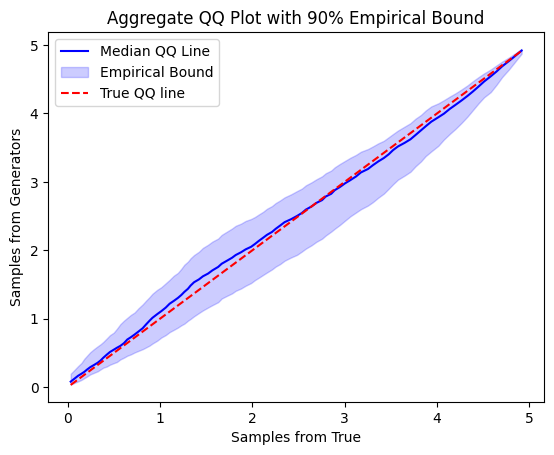}
        \caption*{(b) Subject 2: $9.8\%$ Quantile of risk score}
    \end{subfigure}
    \begin{subfigure}{.5\textwidth}
        \centering
        \includegraphics[width=\linewidth]{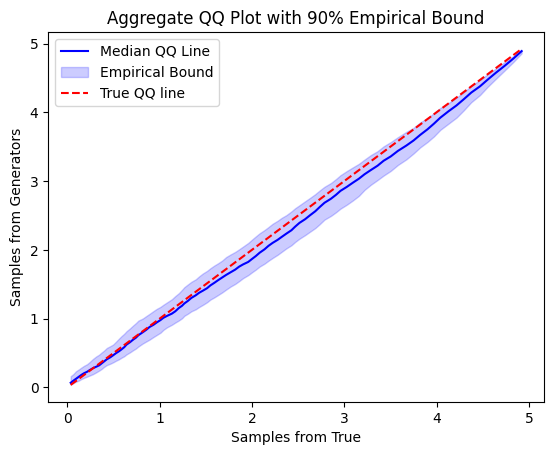}
        \caption*{(c) Subject 3: $39.2\%$ Quantile of risk score}
    \end{subfigure}%
    \begin{subfigure}{.5\textwidth}
        \centering
        \includegraphics[width=\linewidth]{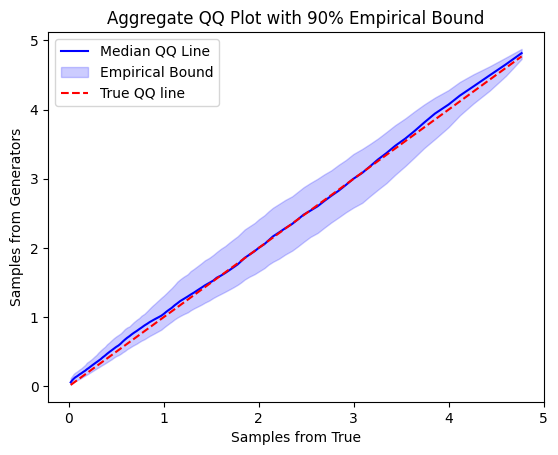}
        \caption{Subject 4: $85.5\%$ Quantile of risk score}
    \end{subfigure}
    \caption{QQ plot of true conditional samples ($x$-axis) and generated samples ($y$-axis) for PH Model, C=5, $N=4000$, $d=100$ : $(5\%,95\%)$ empirical bound for Test Subject 1 to Subject 4.}
    \label{fig:phc5n4000d100qq}
\end{figure}

\begin{figure}[!htbp]
    \centering
    \begin{subfigure}{.5\textwidth}
        \centering
        \includegraphics[width=\linewidth]{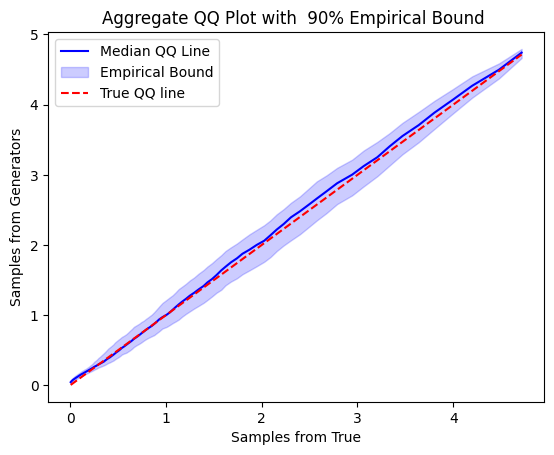}
        \caption{Subject 1: Random sampling}
    \end{subfigure}%
    \begin{subfigure}{.5\textwidth}
        \centering
        \includegraphics[width=\linewidth]{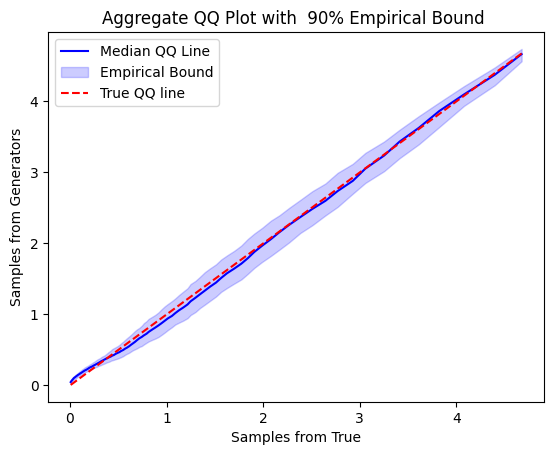}
        \caption{Subject 2: $9.8\%$ Quantile of risk score}
    \end{subfigure}
    \begin{subfigure}{.5\textwidth}
        \centering
        \includegraphics[width=\linewidth]{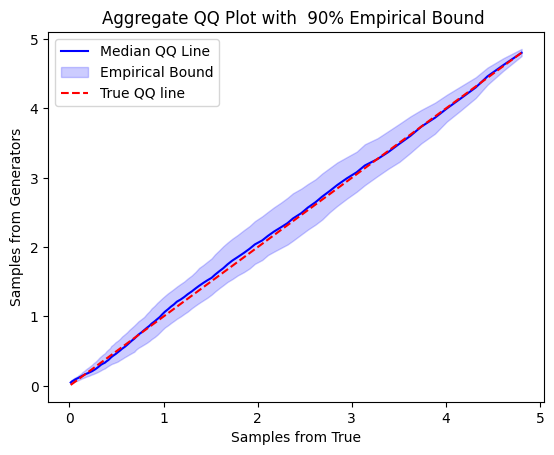}
        \caption{Subject 3: $39.2\%$ Quantile of risk score}
    \end{subfigure}%
    \begin{subfigure}{.5\textwidth}
        \centering
        \includegraphics[width=\linewidth]{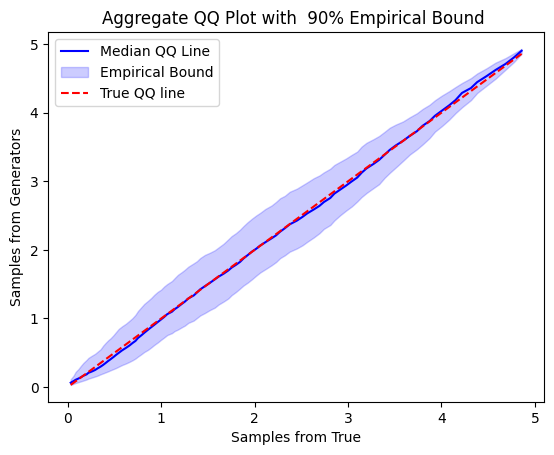}
        \caption{Subject 4: $85.5\%$ Quantile of risk score}
    \end{subfigure}
    \caption{QQ plot of true conditional samples ($x$-axis) and generated samples ($y$-axis) for PO Model, C=5, $N=4000$, $d=100$ : $(5\%,95\%)$ empirical bound for Test Subject 1 to Subject 4.}
        \vspace{-5mm}
    \label{fig:poc5n4000d100qq}
\end{figure}

% \paragraph{Result evalaution} Mean square root of the aggregated squared error (MRASE) 

% \begin{equation}\label{MSE}
%  (\frac{1}{N}  \sum_{i=1}^{N} \sum_{j=1}^{|\mathcal{T}|} (S(t_j|X_i)-\hat{S}(t_j|X_i))^2)^{1/2}
% \end{equation}

% Point out that C-index and Brier are not sensitive. Reference \citep{Gen2014wsvm}

% Figure 7. PO Model, C=5, N = 4000, d = 5  Subject 2 and Subject 3 shows the bias. RSF is better than SCENE. 

% Why High-diemnsional is better than Low-dimensional? This is especially for PO model. For variable procedure, the average selection number of variable is 2 with $100\%$ Power for true varaible. Meanwhile, the PH model is 7. 

\newpage

\section{A Real Data Example}

We applied SCENE to a real-world dataset from the Molecular Taxonomy of Breast Cancer International Consortium (METABRIC) study, accessible via the ``pycox" package in Python. The METABRIC dataset included gene expression and clinical features for 1,904 study participants. The outcome variable, time to death due to breast cancer, was right-censored for 801 (42\%) participants. We considered nine covariates, including four gene expressions (MKI67, EGFR, PGR, and ERBB2) and five clinical features (hormone treatment (HT), radiotherapy (RT), chemotherapy (CT), ER-positive status (ER-P), and age) and aimed to estimate breast cancer survival probabilities conditioned on these nine covariates.

% We compared results obtained from SCENE with those from KM estiamtor, DeepSurv and random survival forests in the following three aspects: (1) the reliability of SCENE, assessed by comparing its average with the KM estimator, and its flexibility in deriving survival functions conditioned on various conditions (2) the prediction accuracy of survival probabilities as measured by the Concordance index (C-index) \citep{harrell1982cindex}, (3) the ability to identify true signals when data were augmented by noise variables.

    % Before Dr. Wang revised: {\color{green} The real data analysis consists of three parts: (1) assessing its accuracy in ranking survival probabilities, (2) demonstrating its effectiveness in detecting heterogeneity of survival functions in subpopulations, {\color{red} as written, the purpose is not clear here} and (3) examining its ability to identify important variables and enhance survival estimation under scenarios with data augmented by additional noise variables.}
    
    First, we obtained conditional survival function estimates using SCENE from the full databset, and compared their average to the KM estimate.  Results are presented in Figure \ref{fig:reliability} (a). Also shown (in dotted blue) are conditional survival function estimates corresponding to five randomly sampled individuals. 
    The KM estimate and the average of conditional survival function estimates from SCENE in general align well. %Beyond this point, the estimated survival probability from KM drops to $0$ steeply, as expected due to its inherent nature. In contrast, SCENE's population-level survival probabilities also decline to $0$, but do so more smoothly. Additionally, the individual survival probabilities from SCENE exhibit various patterns, with some dropping rapidly while others decline more gradually. This demonstrates both the reliability of SCENE estimates and its ability to capture the variability in survival probabilities within the population.
To further illustrate SCENE's ability to estimate conditional survival functions, we obtained survival function estimates conditional on specific levels of each of the two covariates, MKI67 gene expression (low/high) and age (young/old). Low and high gene expression levels were defined as values below the $25\%$ quantile and above the $75\%$ quantile of gene expression levels, respectively. Similarly, young and old age groups were defined by ages below $25\%$ quantile and above the $75\%$ quantile of the age distribution. As shown in Figure \ref{fig:reliability} (b), survival probabilities for the young age group were higher compared to the old age group, while no substantial differences were observed between the low and high MKI67 gene expression groups.

%can also estimate survival probabilities conditional on specific covariate levels. For instance, let $\mathcal{I} = \{i : x_i \in A \}$, where $A$ represents a specific condition. Then survival probailities conditioned on $A$, $S( \cdot | X \in A)$, can be estimated by $\frac{1}{|\mathcal{I}|} \sum_{i \in \mathcal{I}} \hat{S}(\cdot | x_i)$. For example, $A$ could represent conditions such as low or high gene expression levels, or young versus old age groups. Figure \ref{fig:reliability} (b) illustrates survival probabilities conditioned on groups defined by MKI67 gene expression levels and age. Low and high gene expression levels are defined as values below the $25\%$ quantile and above the $75\%$ quantile of gene expression levels, respectively. Similarly, young and old age groups are defined by ages below $25\%$ quantile and above the $75\%$ quantile of the age distribution. As shown in Figure \ref{fig:reliability} (b), survival probabilities for the young age group are higher compared to the old age group, while no substantial differences are observed between the low and high MKI67 gene expression groups.

\begin{figure}[!htbp]
    \centering
    \begin{subfigure}{.5\textwidth}
            \caption{}
        \centering
        \includegraphics[width=\linewidth]{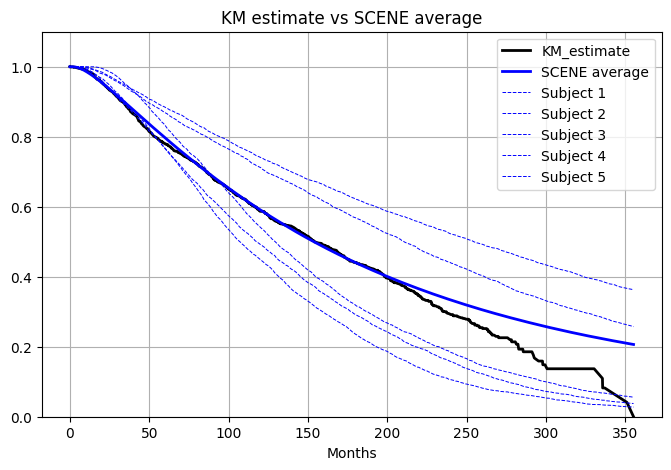}
    \end{subfigure}%
    \begin{subfigure}{.5\textwidth}
            \caption{}
        \centering
        \includegraphics[width=\linewidth]{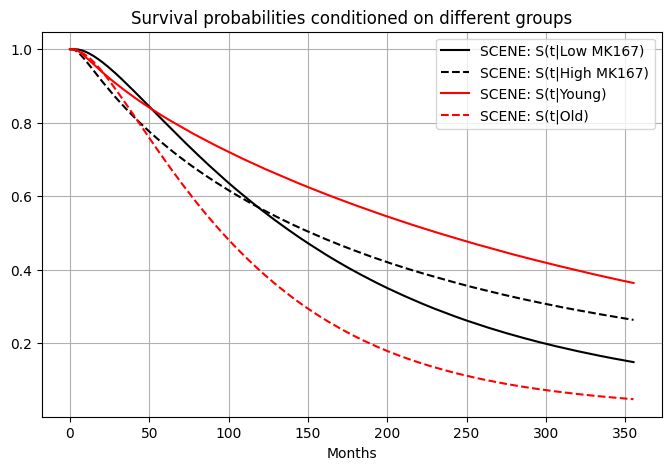}
    \end{subfigure}
    \caption{(a) Comparision of population level survival probabilities from KM estimator and SCENE  (b) Estimated survival probabilities conditioned on low, and high MK167 gene expression level, and young and old age groups}
    \label{fig:reliability}
\end{figure}

Second, we compared SCENE with DeepSurv and random survival forests, using the Concordance index (C-index) \citep{harrell1982cindex} as the evaluation metric for assessing the accuracy of survival probabilities. %{\color{blue} Since the true conditional survival functions are unknown for real data, the C-index evaluates the performance of estimator by calculating the proportion of pairs where the order of predicted survival probabilities and observed survival times are concordant. A higher C-index value indicates a larger proportion of concordant pairs, indicating better model accuracy. }
The C-index evaluates prediction performance by measuring the proportion of concordant pairs between predicted survival probabilities and observed survival times, with higher values indicating greater accuracy.
The C-index values were calculated using 5-fold cross-validation, averaged across test datasets, and are reported in Table \ref{Table:cindex-synthetic}. SCENE achieved the highest C-index value among the three methods, with the difference being especially noticeable when compared to Deepsurv. 

%comparable or slightly higher C-index values than random survival forests and DeepSurv.  %within the range of random variation observed in the cross-validation.

Lastly, we investigated the impact of variable selection on prediction accuracy by augmenting the METABRIC dataset with synthetic noise variables. Specifically, we added 50, 100, and 500 independent random variables sampled from a uniform distribution $[-1,1]^{p_n}$, where $p_n$ denotes the dimension of the random variables. This resulted in three synthetic datasets: Metabric-50, Metabric-100, and Metabric-500. The performance of SCENE with variable selection using variable importance discussed in Section \ref{subsection:variableselection}, random survival forests, and DeepSurv was evaluated using 5-fold cross-validation, with the C-index results summarized in Table \ref{Table:cindex-synthetic}. As shown in the table, SCENE consistently achieved the highest C-index values compared to RSF and Deepsurv.

\begin{table}[!htbp]
\caption{Comparison of the C-index on METABRIC and augmented METABRIC dataset with added noise variables, where the C-index and standard deviation (in parentheses) were calculated using 5-fold cross-validation.}
\label{Table:cindex-synthetic} %\vspace{-0.2in}
\begin{center}
%\begin{adjustbox}{width=1.0\textwidth}
\begin{tabular}{cccc} \toprule
      Dataset & RSF & Deepsurv & SCENE \\
      \midrule
    Metabric & 0.6409(0.0133) & 0.6257(0.0328) & {\bf 0.6451(0.0124)} \\
    \hline
      Metabric-50 & 0.6189(0.0197) & 0.5479(0.0352) & {\bf 0.6331(0.0160)} \\
      Metabric-100 & 0.6160(0.0185) & 0.5311(0.0193) & {\bf 0.6313(0.0065)} \\
      Metabric-500 & 0.5811(0.0259) & 0.5126(0.0171) & {\bf 0.6064(0.0121)} \\
      \bottomrule  
\end{tabular}
 %\end{adjustbox}
%\vspace{-0.2in}
\end{center}
\end{table}

    SCENE's superior performance in achieving high C-index values can be attributed to its ability to accurately identify key variables that influence survival probability while ignoring irrelevant noise variables. To investigate this capability, we calculated the variable importance, which is available for both random survival forests and SCENE. Covariates were ranked based on their importance values, as summarized in Table \ref{Table:Variable_Selected}.
    
    Age, ERBB2 gene expression, and the feature ER-positive consistently emerged as significant predictors for survival probability. However, as the number of added noise variables increased, identifying these critical variables became challenging, with random survival forests failing to detect them in Metabric-100 and Metabric-500 datasets. In contrast, SCENE consistently identified these important variables across all settings, demonstrating its robust performance in complex, noise-affected scenarios.
% We illustrate the  the finite sample performance of the proposed estimators by simulation. We
% first consider scenarios with a single treatment decision time point at baselin
% 0.7564

% Subjects were randomized to four treatment groups with equal probability: zidovudine (ZDV) monotherapy, ZDV plus didanosine (ddI), ZDV plus
% zalcitabine (zal), and ddI monotherapy. A primary composite endpoint of interest is the time
% to having a larger than $50\%$ decline in the CD4 count, or progressing to AIDS, or death,
% whichever comes first. 

% For each subject, there are 12 baseline clinical covariates; preliminary analysis results showed that Karnofsky score (Karnof), baseline CD4 count (CD40), and age (Age) are three important risk predictors and may have interaction effects with treatments

\begin{table}[htbp]
\caption{Comparison of Top 5 Selections from RSF and SCENE for Metabric, Metabric-50, Metabric-100, and Metabric-500 Data from the first fold. }
\label{Table:Variable_Selected}
\vspace{-0.2in}
\begin{center}
\begin{tabular}{cccccccccccc} 
\toprule
&   \multicolumn{2}{c}{Metabric} & & \multicolumn{2}{c}{Metabric-50} & & \multicolumn{2}{c}{Metabric-100} & & \multicolumn{2}{c}{Metabric-500} \\  
\cline{2-3} \cline{5-6} \cline{8-9} \cline{11-12}  
Method &  RSF &  SCENE & & RSF  & SCENE & & RSF  & SCENE & & RSF  & SCENE \\ 
\midrule
Top-1  &  Age & Age & & ERBB2 &Age & &CT& Age && N-435&ERBB2\\  
Top-2  &  CT & ERBB2 & & N27& N-49 &&N-26 &CT &&N-224&Age\\
Top-3  &  EGFR & HT& & N-41& ERBB2 && N-98 &ERBB2&& CT&ER-P\\  
Top-4  &  ER-P & ER-P & & N-49& N-47&&N-19&MKI67 && N-276&MKI67\\
Top-5  &  HT & PGR & & N-37& N-6& &N-47 &EGFR&&N-101&EGFR\\
\bottomrule
\end{tabular}
\end{center}
\end{table}

%, the Study to Understand Prognoses Preferences Outcomes and Risks of Treatment (SUPPORT)
%and the AIDS Clinical Trials Group (ACTG) Study \citep{hammer1996actg}.

% \begin{table}[!htbp]
% \caption{Comparison of the C-index on real data analysis, where the C-index and standard deviation, given in the parentheses, were calculated by 5-cross validation }
% \label{Table:cindex} %\vspace{-0.2in}
% \begin{center}
% %\begin{adjustbox}{width=1.0\textwidth}
% \begin{tabular}{cccc} \toprule
%       Dataset & RSF & Deepsurv & SCENE \\
%       \midrule
%       Metabric & 0.6409(0.0133) & 0.6257(0.0328) & {\bf 0.6451(0.0124)} \\
%       %Support &  {\bf 0.6157(0.0052)} & 0.6038(0.0076) &  0.6025(0.0073)\\
%       % ACTG & 0.6666(0.0190) & 0.6442(0.0267) &  {\bf 0.6715(0.0125)} \\
% \bottomrule  
% \end{tabular}
%  %\end{adjustbox}
% %\vspace{-0.2in}
% \end{center}
% \end{table}

\section{Discussion}

In this paper, we have developed SCENE, a novel and flexible method for estimating conditional survival functions for right-censored time-to-event data. SCENE leverages the self-consistent equations for the conditional survival functions, representing a new approach to deep learning-based survival analysis. Unlike traditional methods, SCENE does not rely on parametric assumptions, proportional hazards assumptions, discrete-time assumptions, or estimation via partial likelihood. This positions SCENE as a unique non-parametric estimation framework, complementing the Kaplan-Meier estimator by integrating deep learning techniques.

We adopted a min-max optimization framework to train SCENE, enabling it to identify the survival function that satisfies weighted self-consistent equations for all possible non-negative, bounded weight functions, $\phi(\cdot)$. Furthermore, we established theoretical guarantees to support the proposed min-max optimization method. This framework not only facilitates the training of SCENE but also holds promise for broader applications in solving problems that involve infinitely many equations. 

We conducted a comprehensive set of experiments to evaluate SCENE’s performance. Across both real and simulated datasets, SCENE consistently outperformed or matched competing methods. In particular, by incorporating variable selection using variable importance \citep{SunSLiang2021}, SCENE demonstrated its effectiveness in handling high-dimensional covariates.

Here we only considered the right-censoring case. For other scenarios, such as interval censoring, our approach can be generalized by extending the self-consistent equation for interval censored data  \citep{turnbull1976interval-self}. This warrants future research.

%Moreover, exploring the application of variable importance within SCENE to broader use cases would be a compelling direction for further research.

\paragraph{Availability}
The code that implements the SCENE method can be found at \url{https://github.com/sehwankimstat/SCENE}.

\paragraph{Acknowledgments and Disclosure of Funding} This work was supported by grant R01 AI170254 from the National Institute of Allergy and Infectious Diseases (NIAID). 

\bibliography{reference}
\bibliographystyle{apalike} 

\newpage

\appendix

\section{Proofs}

% \newtheorem{proprestate}{Proposition}
% \newtheorem{proprestate2}{Proposition}

% \renewcommand{\theproprestate}{\ref{uniquesol}}
% \renewcommand{\theproprestate2}{\ref{constestimator}}

% \begin{proprestate}
    
% \end{proprestate}

% \begin{proprestate2}
    
% \end{proprestate2}

\subsection{Proof of Proposition 1}

    \begin{proof} 
    
    It's obvious that $S_T^*$ satisfies (\ref{const:2}). Suppose there exists another solution, $S_T^a(t|x)$. This implies the existence of $\mathcal{T}_D = \{t : S_T^*(t|x) \neq S_T^a(t|x)\} \subseteq \mathcal{T}$, such that $\Pi(\mathcal{T}_D) > 0$, where $\Pi$ is a some probability measure on $(\mathbb{R}^+, \mathcal{R}^+)$, with $\mathcal{R}^+$ being the Borel $\sigma$-algebra. For $u \in \mathcal{T}_D$, either $S_T^*(u|x) > S_T^a(u|x)$ or $S_T^*(u|x) < S_T^a(u|x)$, allowing us to partition $\mathcal{T}_D$ as $\mathcal{T}_D = \mathcal{T}_{D,P} \cup \mathcal{T}_{D,N}$, where:
    \[
    \mathcal{T}_{D,P} = \{t : S_T^*(t|x) > S_T^a(t|x)\}, \quad 
    \mathcal{T}_{D,N} = \{t : S_T^*(t|x) < S_T^a(t|x)\}.
    \]
    Since $\Pi(\mathcal{T}_D) > 0$, $\mathcal{T}_D$ can be expressed as a union of open intervals as $\mathcal{T}_D = \bigcup_{k=0}^K (r_k, r_{k+1})$, with $r_k < r_{k+1}$ for some $K \in \mathbb{N}^+$. Thus, $\mathcal{T}_D$ must fall into one of two cases. 

    \noindent \textbf{Case (i)} $\Pi(\mathcal{T}_{D,P}) = 0$ (or $\Pi(\mathcal{T}_{D,N}) = 0$)
    
    This implies that $\frac{S_T^*(t|x)}{S_T^a(t|x)} < 1$ for all $t \in \mathcal{T}_D$ and $S_T^*(t|x)=S_T^a(t|x)$ for $t\in \mathcal{T} \setminus \mathcal{T}_D$. So, for any $t_0 > r_0$, we can derive:
    \[
    E_C \left[ \frac{S_T^*(C|x)}{S_T^a(C|x)} I(C \leq t_0) \right] < (1 - S_C^*(t_0|x)),
    \]
    which contradicts the self-consistent condition of $S_T^a(t_0|x)$. Similarly, for $\Pi(\mathcal{T}_{D,N}) = 0$, analogous reasoning leads to a contradiction.

    \noindent \textbf{Case (ii)} $\Pi(\mathcal{T}_{D,P}) > 0$, $\Pi(\mathcal{T}_{D,N}) > 0$
    
    In this scenario, we note that $(r_0, r_1) \subseteq \mathcal{T}_{D,P}$ or $(r_0, r_1) \subseteq \mathcal{T}_{D,N}$. Regardless of whether $S_T^*(t|x) > S_T^a(t|x)$ or $S_T^*(t|x) < S_T^a(t|x)$ in $t\in (r_0, r_1)$, the $S_T^a(t|x)$ does not solve the self-consistent equation. 

    \end{proof}
    
\subsection{Proof of Proposition 2}

    \begin{proof} $(\Rightarrow$) Let define $\mathcal{X}_P=\{x: S(t|x)> S_{T}^*(t|x)S_{C}^*(t|x)+E_C[
        \frac{S_T^*(C|x)}{S_{T}(C|x)} I(C\leq t) |x]S_T(t|x)\}$ and  $\mathcal{X}_N=\{x: S(t|x)< S_{T}^*(t|x)S_{C}^*(t|x)+E_C[
        \frac{S_T^*(C|x)}{S_{T}(C|x)} I(C\leq t) |x]S_T(t|x)\}$, then we can consider the function $\phi_P(x) = B$ for $x \in \mathcal{X}_P$ and $\phi_P(x) = 0$ otherwise. If $\mathcal{X}_P$ has a positive probability measure in the covariate space, then $D^I(t, S, \phi_P)$ has positive value which contradicts the condition $D^I(t, S, \phi_P) = 0$.  By a similar argument, we can define $\phi_N$ such that it contradicts $D^I(t, S, \phi_N) = 0$ if $\mathcal{X}_N$ has a positive measure. Consequently, for any probability measure on $\mathcal{X}$, the measures of $\mathcal{X}_P$ and $\mathcal{X}_N$ must both be zero.
            
        $(\Leftarrow$) If $S(t|X)$ solve the Equation \ref{const:2} almost surely, it is trivial that $D^I(t,S,\phi)=0$ for any $\phi$ by its definition.  
    \end{proof}

\subsection{Proof of Theorem \ref{thm:1}}

    \begin{proof} 
    $(\Rightarrow)$ Suppose $\underset{\phi \in \Phi_B}{\max} \ C(S, \phi) = 0$. This implies $C(S, \phi) = 0$ for all $\phi \in \Phi_B$, so that $D^I(t, S, \phi) = 0$ for all $t \in \mathcal{V}$ and $\phi \in \Phi_B$. Thus, for $S$ such that $\underset{\phi \in \Phi_B}{\max} \ C(S, \phi) = 0$, the $S$ satisfies Equation (\ref{const:2}) for $t \in \mathcal{V}$ and $x \in \mathcal{X}$ almost surely. By Proposition \ref{uniquesol}, it follows that $S_T(t|x) = S_T^*(t|x)$ for all $t \in \mathcal{V}$ and $x \in \mathcal{X}$ almost surely.
    
    $(\Leftarrow)$ Now suppose $S_T(t|x) = S_T^*(t|x)$ for all $t \in \mathcal{V}$ and $x \in \mathcal{X}$. Then $D^I(t, S, \phi) = 0$ for all $t \in \mathcal{V}$ and any $\phi$, which trivially implies that $C(S, \phi) = 0$ for all $\phi \in \Phi_B$, resulting $\underset{\phi \in \Phi_B}{\max} \ C(S, \phi) = 0$
    \end{proof}

\subsection{Proof of Theorem \ref{thm:consistency}}

First, we introduce and restate the terms necessary for proving the theorem. Let $Z_i$ denote the triplet of random variables consisting of the observed survival time, censoring indicator, and the covariate. More specifically, we define $z_i = \{\tilde{t}_i, \delta_i, x_i\}$ and introduce the following terms:
\begin{enumerate}
    \item $l(t,z,S,\phi)=S(t|x)\phi(x)-(I(\tilde{t}>t)\phi(x)+\frac{I(\delta=0)}{S(\tilde{t}|x)}I(\tilde{t}<t)S(t|x)\phi(x))$
    \item $D(t,S,\phi)=D^I(t,S,\phi)=E_Z[l(t,Z,S,\phi)]$
    \item $C(S,\phi)=E_V[D(V,S,\phi)^2]$
    \item $C_{M,N}(S,\phi)=\frac{1}{M}\sum_{m=1}^M (\frac{1}{N}\sum_{i=1}^N(l(V_m,Z_i,S,\phi)))^2$ 
\end{enumerate}
Additionally, we define $C(S)=\underset{\phi\in \Phi_B}{\max} C(S,\phi)$, $C_{M,N}(S)=\underset{\phi\in\Phi_B}{max}C_{M,N}(S,\phi)$ and $\bar{C}_{M,N}(S)=\underset{\phi\in \Phi^B_{\zeta}}{max}C_{M,N}(S,\phi)$.

    For $S^* = \arg \underset{S \in \mathcal{S}}{\min} C(S)$, we impose Lipschitz continuity and curvature conditions around the solution $S^*$, as in Assumption \ref{asm:lip_curvarture}. These conditions are standard \citep{Farrell_2021}. Furthermore, we assume regularity conditions for both $S^*$ and $\phi_{\hat{S}^{\text{SCENE}}}$ such that $\phi_{\hat{S}^{\text{SCENE}}} = \arg \underset{\phi \in \Phi_B}{\max} C_{M,N}(\hat{S}, \phi)\in \Phi_B$, as in Assumption \ref{asm:sobolev}. Since $S^*$ is a monotone decreasing function, there exists an inverse function $G^*(u, x)$ satisfying $P_U(G^*(U, x) > t) = S^*(t | x)$. Consequently, $S^* \in \mathcal{W}^{\beta,\infty}([-1,1]^{p+p_u})$ represents that $G^* \in \mathcal{W}^{\beta,\infty}([-1,1]^{p+p_u})$ with smoothness parameter $\beta$.

    And the function class $\mathcal{F}_{\text{DNN}}$ represents deep neural networks with $L$ layers, where the $l$th layer contains $H_l$ hidden units. For simplicity, we assume $H_l = H$ for all $l = 1, \dots, L$. Let $W$ denote the total number of parameters in the network, and $U$ the total number of hidden units across all layers.

Then, we begin the proof by decomposing the error between $\hat{S}^{\text{SCENE}}$ and $S^*$ into the approximation error and the stochastic error. Subsequently, we derive bounds for each term by utilizing the supporting lemmas as stated in Section \ref{supportlemma}.

\paragraph{Error decomposition}

 Let start with error decomposition by considering the best possible $\bar{S}$ that can be approximated in $\mathcal{S}_w$ the true $S^*$ as $\bar{S}=\arg\underset{S\in\mathcal{S}_{w}}{min}\|S-S^*\|_{\infty}$, where $\epsilon_S=\|\bar{S}-S^*\|_{\infty}$ Again, as discussed above, this expresses $\bar{G}=\arg\underset{G\in\mathcal{F}_{DNN}}{\min}\|G-G^*\|_{\infty}$. And for notation simplicity, let $\hat{S}$ denote the $\hat{S}^{\text{SCENE}}$.

% \paragraph{Error decomposition} From the smoothness assumption,

% Note that for any $S\in \mathcal{S}$, $\|S\|_{\infty}\leq 1$, define 

\begin{alignat*}{2}
    c_{1,C}\|\hat{S}-S^*\|^2_{L_2(t,x)|_{\mathcal{V}}}&\leq C(\hat{S})-C(S^*) \\
    & \leq C(\hat{S})-C(S^*)-\bar{C}_{M,N}(\hat{S})+\bar{C}_{M,N}(\bar{S})\\
    & = C(\hat{S})-C(S^*)-(C_{M,N}(\hat{S})-C_{M,N}(S^*)) && \quad \cdots \quad (I) \\
    & \ \ + C_{M,N}(\hat{S})-\bar{C}_{M,N}(\hat{S}) && \quad \cdots \quad (II) \\
    & \ \ +  \bar{C}_{M,N}(S^*)-C_{M,N}(S^*) && \quad  \cdots \quad (III)\\
    & \ \ +\bar{C}_{M,N}(\bar{S})-\bar{C}_{M,N}(S^*) &&\quad \cdots \quad (IV)
\end{alignat*}

The first inequality follows from the curvature assumption, and the second inequality is due to the fact that $\hat{S} = \arg \underset{S \in \mathcal{S}_{\omega}}{\min} \bar{C}_{M,N}(S)$. Since $\bar{C}_{M,N}(S^*) \leq C_{M,N}(S^*)$, the term $(III)$ is negative. Thus, we only need to bound $(I)$, which is related to the stochastic error, and $(II)$ and $(IV)$, which are related to the approximation error. And let $E_{Z}$ denote the expectation with respect to $Z_1, \dots, Z_N$ and $E_{V}$ the expectation with respect to $V_1, \dots, V_M$.

\paragraph{Bound $(IV)$} We will apply the Bernstein inequality twice to bound the mean of a function with respect to two random variables, $Z$ and $T$. First, we apply the Bernstein inequality with respect to the random variable $Z$ for each $V_m$, $m = 1, \dots, M$. Then, with probability at least $1 - \exp(-\gamma)$, the following holds:
\[
\begin{split}
    \bar{C}_{M,N}(\bar{S})-\bar{C}_{M,N}(S^*)&\leq | \max_{\phi \in \Phi^B_{\zeta}} C_{M,N}(\bar{S},\phi)-  \max_{\phi \in \Phi^B_{\zeta}} C_{M,N}(S^*,\phi)|\leq  \max_{\phi \in \Phi^B_{\zeta}} |C_{M,N}(\bar{S},\phi)-C_{M,N}(S^*,\phi) |        \\
    &=\max_{\phi\in\Phi_B}\left|\frac{1}{M}\sum_{m=1}^M \left\{ \frac{1}{N}\sum_{i=1}^Nl(V_m,X_i,\bar{S},\phi) \right\}^2-\frac{1}{M}\sum_{m=1}^M \left\{\frac{1}{N}\sum_{i=1}^Nl(V_m,X_i,S^*,\phi)\right\}^2\right|\\
    &\leq 6\|\phi\|_{\infty}\max_{\phi\in\Phi_B}\frac{1}{M}\sum_{m=1}^M \left|\frac{1}{N}\sum_{i=1}^Nl(V_m,X_i,\bar{S},\phi)-\frac{1}{N}\sum_{i=1}^Nl(V_m,X_i,S^*,\phi)  \right| \\
    &\leq 6\|\phi\|_{\infty}\max_{\phi\in\Phi_B}\frac{1}{M}\sum_{m=1}^M \left|D(V_m,\bar{S},\phi)-D(V_m,S^*,\phi)\right|\\
    &\hspace{5cm} +\sqrt{\frac{2C_S^2\|\bar{S}-S^*\|_{\infty}^2\gamma}{N}}+\frac{6\|\phi\|_{\infty}\gamma}{3N}.
\end{split}
\]

The inequality on the third line follows from Lemma \ref{uniformlbound}, while the inequality on the fourth line follows from Assumption \ref{asm:lip_curvarture}, $\text{Var}_X(l(V_m, X_i, \bar{S}, \phi) - l(V_m, X_i, S^*, \phi)) \leq C_S^2 \|\bar{S} - S^*\|_{\infty}^2$ for all $m = 1, \dots, M$.

Similarly, we apply the Bernstein inequality with respect to the random variable $V$, using the fact that $Var(D(T_m,\bar{S},\phi)-D(T_m,S^*,\phi))\leq E[|D(V_m,\bar{S},\phi)-D(V_m,S^*,\phi)|^2]\leq C_S^2\|\bar{S}-S^*\|^2_{\infty}$. Then, with probability at least $1 - \exp(-\gamma)$, the following inequality holds:
\[
\begin{split}
    \frac{1}{M}\sum_{m=1}^M |D(V_m,\bar{S},\phi)-D(V_m,S^*,\phi)|\leq E_V\left\{\left|D(V,\bar{S},\phi)-D(V,S^*,\phi)\right|\right\}+\sqrt{\frac{2C_S^2\|\bar{S}-S^*\|^2_{\infty}\gamma}{M}}+\frac{6\|\phi\|_{\infty}\gamma}{3M}.
\end{split}
\]
All together, we can derive that, with probability at least $1 - 2\exp(-\gamma)$, the following inequality holds for some constants $C_{A,S}$, $C_{1,A}$, $C_{2,A}$, and $C_{3,A}$:
\begin{equation}\label{bound4}
\begin{split}
    \bar{C}_{M,N}(\bar{S})-\bar{C}_{M,N}(S^*)&\leq 6\|\phi\|_{\infty}\max_{\phi}E_T[D(T,\bar{S},\phi)-D(T,S^*,\phi^*)]\\
    & + 6\|\phi\|_{\infty}\sqrt{\frac{2C_S^2\|\bar{S}-S^*\|^2_{\infty}\gamma}{M}}+\frac{36\|\phi\|_{\infty}^2\gamma}{3M} \\
    & +\sqrt{\frac{2C_S^2\|\bar{S}-S^*\|_{\infty}^2\gamma}{N}}+\frac{6\|\phi\|_{\infty}\gamma}{3N}\\
    &\leq C_{1,A}\epsilon_S^2+C_{2,A}\epsilon_S(\frac{1}{\sqrt{M}}+\frac{1}{\sqrt{N}})\sqrt{\gamma}+C_{3,A}(\frac{1}{M}+\frac{1}{N})\gamma\\
    &\leq C_{A,S}\left\{(\frac{1}{\sqrt{M}}+\frac{1}{\sqrt{N}})\sqrt{\gamma} +\epsilon_S\right\}^2 \\
    % & \leq C_A((\frac{1}{\sqrt{M}}+\frac{1}{\sqrt{N}})\sqrt{\gamma} +n^{-\frac{\beta}{2(\beta+p+nz)}})^2
\end{split}
\end{equation}

\paragraph{Bound (II)} Note that $C_{M,N}(\hat{S}) - \bar{C}_{M,N}(\hat{S})$ represents the approximation error from restricting the function space from $\Phi_B$ to $\Phi_{\zeta}^B$. Let $\phi_{\hat{S}} = \arg \underset{\phi \in \Phi_B}{\max} C_{M,N}(\hat{S}, \phi)$ and denote $\bar{\phi} = \arg \underset{\phi \in \Phi_{\zeta}^B}{\min} \|\phi - \phi_{\hat{S}}\|_{\infty}$ with $\epsilon_{\phi} = \|\bar{\phi} - \phi_{\hat{S}}\|_{\infty}$. Then, we have:
\begin{equation}\label{bound3}
    \begin{split}
\underset{\phi\in \Phi_B}{\max}C_{M,N}(\hat{S},\phi)-\underset{\phi\in \Phi^B_{\zeta}}{\max}C_{M,N}(\hat{S},\phi)&=\underset{\phi\in\Phi^B_{\zeta}}{\min} \left\{C_{M,N}(\hat{S},\phi_{\hat{S}})-C_{M,N}(\hat{S},\phi)\right\}\\
&\leq C_{\phi}\underset{\phi\in \Phi^B_{\zeta}}{\min}\|\phi_{\hat{S}}-\phi\|_{\infty}=C_{\phi}\epsilon_{\phi} \\
% & \leq C n^{-\beta/(2\beta+p)}
\end{split}
\end{equation}
where inequality at second line comes from the Assumption \ref{asm:lip_curvarture}. The approximation error $\epsilon_{\phi}$, together with $\epsilon_S$ will be bounded at the end of section. 

\paragraph{Bound (I)} The term (I), $C(\hat{S}) - C(S^*) - (C_{M,N}(\hat{S}) - C_{M,N}(S^*))$, represents the empirical process term, which can be bounded using the Rademacher complexity. First, let us revisit the concept of Rademacher complexity. 

Let $\mathcal{F}$ denote the function class of $f$, where $f: \mathcal{Z} \to \mathbb{R}$, whose capacity we aim to evaluate. Rademacher complexity, roughly speaking, quantifies the supremum of the correlation between random signs $\eta \sim \text{Unif}\{-1,1\}$ and $f(Z)$, where $Z \sim D|_Z$ and $D|_Z$ is the distribution of $Z$. This measures the ability of $f$ to approximate random noise. The Rademacher complexity $\mathcal{R}_n(\mathcal{F})$ and the empirical Rademacher complexity $\hat{\mathcal{R}}_N(\mathcal{F})$ are defined as follows:
\[
    \begin{split}
\mathcal{R}_N(\mathcal{F})&=E_{\eta,Z}[\underset{f\in\mathcal{F}}{sup}\frac{1}{N}\sum_{i=1}^N\eta_i 
f(z_i)] \\
\hat{\mathcal{R}}_N(\mathcal{F})&=E_{\eta}[\underset{f\in\mathcal{F}}{sup}\frac{1}{N}\sum_{i=1}^N\eta_i 
f(z_i)].
    \end{split}
\]
Then, the Rademacher complexity is utilized to bound the difference between the empirical mean $\frac{1}{N} \sum_{i=1}^N f(z_i)$ and the population expectation $E[f(Z)]$. 

In our case, the loss function $l(t, z, S, \phi)$, given $S$ and $\phi$, is of primary interest, as $C(S, \phi)$ is constructed as the expectation of $l(t, z, S, \phi)$. Let consider the function class $\mathcal{L} = \{l=l(\cdot, \cdot, S, \phi),S\in \mathcal{S},\phi\in\Phi_B\}$. Since $|l| \leq 3\|\phi\|_{\infty}$ and $V[l] \leq E[l^2] \leq 9\|\phi\|_{\infty}^2$, by applying Lemma \ref{Symmetrization}, we can establish that, with probability at least $1 - 2\exp(-\gamma)$ for $m = 1, \dots, M$, the following inequality holds:
\begin{equation}\label{stochastic:1}
    \begin{split}
    \frac{1}{N}\sum_{i=1}^nl(V_m,Z_i,S,\phi) &\leq D(V_m,S,\phi) + 6\hat{\mathcal{R}}(\mathcal{L})+\sqrt{\frac{6\|\phi\|_{\infty}\gamma}{N}}+\frac{207\|\phi\|_{\infty}^2\gamma}{3N}\\
    & \leq D(V_m,S,\phi)+\frac{36\|\phi\|_{\infty}+\sqrt{6\|\phi\|_{\infty}\gamma}}{\sqrt{N}}+\frac{69\|\phi\|_{\infty}^2\gamma}{N} \\
\end{split}
\end{equation}
, where the second inequality follows from the fact that $\hat{\mathcal{R}}(\mathcal{L}) = E_{\eta} \left[\underset{l \in \mathcal{L}}{\sup} \frac{1}{N} \sum_{i=1}^N \eta_i l(V_m, Z_i,S,\phi)\right] \leq \frac{6\|\phi\|_{\infty}}{\sqrt{N}}$, since the loss function $l(\cdot, \cdot, S, \phi)$ is bounded by $\|\phi\|_{\infty}$ by Lemma \ref{uniformlbound}.

Similarly, we can derive
\begin{equation}\label{stochastic:2}   
\frac{1}{M}\sum_{m=1}^M D(V_m,S,\phi)^2 \leq C(S,\phi) + \frac{108\|\phi\|_{\infty}^2+\sqrt{18\|\phi\|_{\infty}^2\gamma}}{\sqrt{M}}+\frac{621\|\phi\|_{\infty}^4\gamma}{M}.
\end{equation}
And by combining Eq (\ref{stochastic:1}), (\ref{stochastic:2}), with probability at least $1-4exp(-\gamma)$, with for some constants $C_1,\dots,C_6$, we can get following inequality:
\begin{equation}\label{bound1}
    C_{M,N}(S,\phi)\leq C(S,\phi)+C_{1,S}(\frac{1}{\sqrt{N}}+\frac{\sqrt{\gamma}}{\sqrt{N}}+\frac{\gamma}{N})^2
    +C_{2,S}(\frac{1}{\sqrt{M}}+\frac{\sqrt{\gamma}}{\sqrt{M}}+\frac{\gamma}{M})
\end{equation}
Replacing $S$ with $\hat{S}$ and $S^*$, and taking the maximum with respect to $\phi \in \Phi_B$, we can bound term (I) by the factor given in Equation (\ref{bound1}).

\vspace{3mm}

Combining Equations (\ref{bound4}), (\ref{bound3}), and (\ref{bound1}), we conclude that when $M=O(N)$, with probability at least $1 - 6\exp(-\gamma)$, there exists a constant $C'$ independent of $N,M$ and $K$ such that the following inequalities hold:
\[
\begin{split}
    \|\hat{S}-S^*\|_{L_2(t,x)|_{\mathcal{V}}}^2&\leq C_{A,S}\left\{(\frac{1}{\sqrt{M}}+\frac{1}{\sqrt{N}})\sqrt{\gamma} +\epsilon_S\right\}^2 +C_{\phi}\epsilon_{\phi}\\
    &\hspace{4cm}+C_{1,S}(\frac{1}{\sqrt{N}}+\frac{\sqrt{\gamma}}{\sqrt{N}}+\frac{\gamma}{N})^2
    +C_{2,S}(\frac{1}{\sqrt{M}}+\frac{\sqrt{\gamma}}{\sqrt{M}}+\frac{\gamma}{M})\\
&\leq C' \left\{ \epsilon_S(\sqrt{\frac{\gamma}{N}}+\sqrt{\frac{\gamma}{M}})+\epsilon_S^2+\epsilon_{\phi}+N^{-\frac{1}{2}}+M^{-\frac{1}{2}}\right\} \\
&\leq C'\left(K^{-1}N^{-\rho_1}+N^{-\frac{1}{4}-\frac{\beta}{2\beta+p+nz}} +N^{-\frac{2\beta}{2\beta+p+nz}}+N^{-\frac{\beta}{2\beta+p}}+N^{-\frac{1}{2}}+M^{-\frac{1}{2}}+K^{-2}\right)\\
& \leq C'\left(K^{-1}N^{-\rho_1}+N^{-\rho_2}+M^{-1/2}+K^{-2}\right),
\end{split}
\] 
where $\rho_1=\min(\frac{1}{4},\frac{\beta}{2\beta + p + n_z})$, $\rho_2 = \min\left(\frac{1}{4}+\frac{\beta}{2\beta+p+nz},\frac{2\beta}{2\beta + p + n_z}, \frac{\beta}{2\beta + p}\right)$. The third inequality follows from controlling $\epsilon_{S}$ and $\epsilon_{\phi}$ and selecting $\gamma = O(N^{1/4}M^{1/4})$. To bound $\epsilon_{S}$, let consider
\[
\mathcal{S}_{\omega}^{\infty} = \left\{ S(\cdot|x) : S(\cdot|x) = \underset{K\to \infty}{lim} \frac{1}{K} \sum_{k=1}^{K} I(T_{k}(x) > \cdot), \; T_{k}(x) = G_{\omega}(U_k, x) \right\}.
\]
For any $S\in \mathcal{S}_{\omega}^{\infty}$, we can establish the bound
\[
\epsilon_{S}\leq  \left\{ \|\bar{S}-S\|_{\infty}+\|S-S^*\|_{\infty} \right\}. 
\]
The first term, $\|\bar{S} - S\|_{\infty}$, can be controlled by $O\left(\frac{1}{K}\right)$ using the Glivenko–Cantelli theorem. The second term, $\|S - S^*\|_{\infty}$, can be bounded via Lemma \ref{dnnapproximate} as $O\left(N^{-\frac{\beta}{2\beta + p + p_u}}\right)$, under the assumptions $W_S, U_S \asymp N^{(p + p_u)/(2\beta + p + p_u)}$ and $L_S \asymp \log N $. Similarly, $\epsilon_{\phi}$ can be bounded using analogous arguments by $O(N^{-\frac{\beta}{2\beta+p}})$, under the assumptions $W_{\phi}, U_{\phi} \asymp N^{p/(2\beta + p)}$ and $L_{\phi} \asymp \log N$.

\section{Supporting Lemmas}\label{supportlemma}

\begin{lemma}[Theorem 1 from \citep{yarotsky2017} and Lemma 7 from \citep{Farrell_2021}] \label{dnnapproximate} 

There exists a network class $\mathcal{F}_{DNN}$, with ReLU activation, \citep{yarotsky2017}, such tat for any $\epsilon>0:$

(a) $\mathcal{F}_{\text{DNN}}$ approximates the $W^{\beta,\infty}([-1,1]^d)$ in the sense for any $f\in W^{\beta,\infty}([-1,1]^d)$, there exsits a $f_n\in\mathcal{F}_{\text{DNN}}$ such that

\[
\|f_n-f\|_{\infty}\leq \epsilon
\]

(b) $\mathcal{F}_{DNN}$ has $L(\epsilon)\leq C(log(1/\epsilon)+1)$ and $W(\epsilon),U(\epsilon)\leq C \epsilon^{-d/\beta}(log(1/\epsilon)+1)$
    
\end{lemma}

\begin{lemma}[Symmetrization, Theorem 2.1 in \citep{Bartlett_2005} or Lemma 5 in \citep{Farrell_2021}]\label{Symmetrization}
For all $g \in \mathcal{G}$, $|g| \leq G$, and $V[g] \leq V$, then, with probability at least $1 - 2e^{-\gamma}$, we have:
\[
\sup_{g \in \mathcal{G}} \mathbb{E}g - \mathbb{E}_n[g] \leq 6\hat{\mathcal{R}}_n(\mathcal{G}) + \sqrt{\frac{2V\gamma}{n}} + \frac{23G\gamma}{3n},
\]
where $\hat{\mathcal{R}}_n(\mathcal{G})$ is the empirical Rademacher complexity of the class $\mathcal{G}$, $\mathbb{E}g$ denotes the expectation of $g$, and $\mathbb{E}_n[g]$ is the empirical mean of $g$ over the sample of size $n$.
\end{lemma}

\begin{lemma}[Uniform bound of loss function]\label{uniformlbound} For all $t\in \mathbb{R},x\in \mathbb{R}^p$ and function $S$, and observed survival time $\tilde{t}\in \mathbb{R}$ and censoring indicator $\delta$, we have
\[
|l(t,z,S,\phi)|\leq 3\|\phi\|_{\infty} \quad .
\]    
\end{lemma} 

\begin{proof} Since the survival function $S(t | x) \in [0,1] $ and $S(t | x) \leq S(\tilde{t} | x) $ when $ \tilde{t} \leq t$, we can derive the following inequality:
    \[
    \begin{split}
    l(t,z,S,\phi)&=S(t|x)\phi(x)-(I(\tilde{t}>t)\phi(x)+\frac{I(\delta=0)}{S(\tilde{t}|x)}I(\tilde{t}<t)S(t|x)\phi(x)) \\
                  &\leq \|\phi\|_{\infty}+\|\phi\|_{\infty}+\frac{1}{S(t|x)}S(t|x)\phi(x)\leq 3\|\phi\|_{\infty}.
    \end{split}
    \]
\end{proof}

\section{Experimental settings}

We set \(K = 400\), \(p_u = 5\), and use a mini-batch size of 5, with \(U_k \sim \text{Unif}[-1,1]^{p_u}\) for all studies.

\subsection{Simulation Study}

\subsubsection{Low-Dimensional Case}

For the conditional distribution generator, we use the architecture \((p + p_u) - 1000 - 1000 - 1000 - 1\) with ReLU activation, and for \(\phi\), we use the architecture \(p - 1000 - 1000 - 1\) with ReLU activation. The conditional distribution generator was trained using Adam optimization with a learning rate of \(2 \times 10^{-4}\) and parameters \((\beta_1, \beta_2) = (0, 0.9)\). The \(\phi\) model was trained using SGD with a learning rate of \(1 \times 10^{-3}\) and momentum of 0.9, for a total of 50 epochs.

\subsubsection{High-Dimensional Case}

For the conditional distribution generator, we use the architecture \((p + p_u) - 100 - 100 - 100 - 1\) with ReLU activation, and for \(\phi\), we use the architecture \(p - 1000 - 1000 - 1\) with ReLU activation. The conditional distribution generator was trained using Adam optimization with a learning rate of \(2 \times 10^{-4}\) and parameters \((\beta_1, \beta_2) = (0, 0.9)\). The \(\phi\) model was trained using Adam with a learning rate of \(1 \times 10^{-4}\) and parameters \((\beta_1, \beta_2) = (0.5, 0.999)\). We trained SCENE without variable selection until either a total of 120 epochs was reached or the average weight importance for covariates was larger than the weight importance for auxiliary variables. After that, we trained SCENE with variable selection for an additional 20 epochs.

\subsection{Real Data Analysis}

For the conditional distribution generator, we use the architecture \((p + p_u) - 100 - 100 - 1\), trained using the Adam optimizer with a learning rate of \(2 \times 10^{-4}\) and parameters \((\beta_1, \beta_2) = (0, 0.9)\). For \(\phi\), we use the architecture \(p - 1000 - 1000 - 1\), trained using the Adam optimization with a learning rate of \(2 \times 10^{-6}\) and parameters \((\beta_1, \beta_2) = (0.5, 0.999)\). For the synthetic data case, when the dimension of the added noise is 50 or 100, we trained the model without variable selection for a total of 200 epochs. When the dimension of the added noise is 500, the training continued for a total of 300 epochs. After that, we trained for an additional 20 epochs with variable selection.

\newpage

\section{More results for Simulations}

\subsection{PH Model / Moderate Censoring / Low dimensional}

\begin{figure}[!htbp]
    \centering
    \begin{subfigure}{.5\textwidth}
        \centering
        \includegraphics[width=\linewidth]{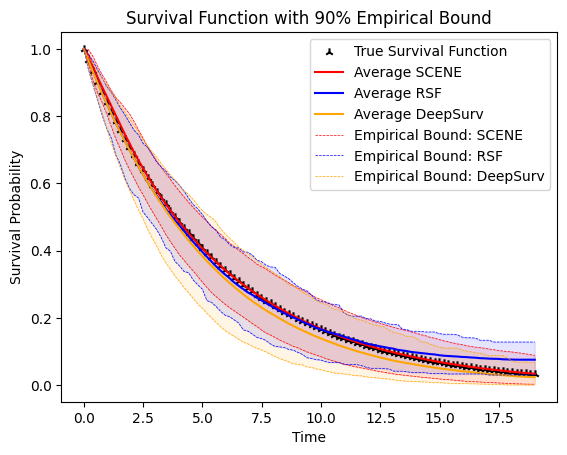}
        \caption{Subject 1: Random sampling}
    \end{subfigure}%
    \begin{subfigure}{.5\textwidth}
        \centering
        \includegraphics[width=\linewidth]{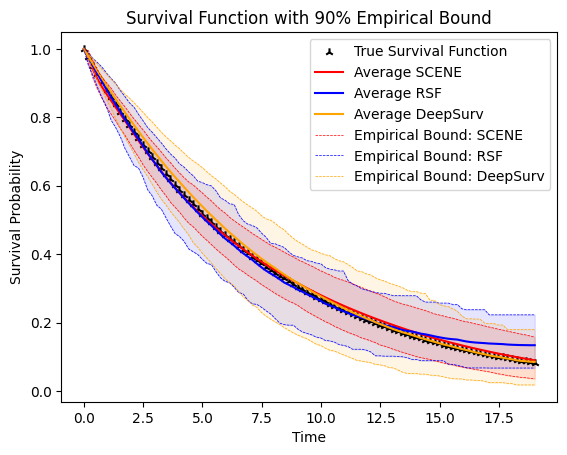}
        \caption{Subject 2: $9.8\%$ Quantile of risk score}
    \end{subfigure}
    \begin{subfigure}{.5\textwidth}
        \centering
        \includegraphics[width=\linewidth]{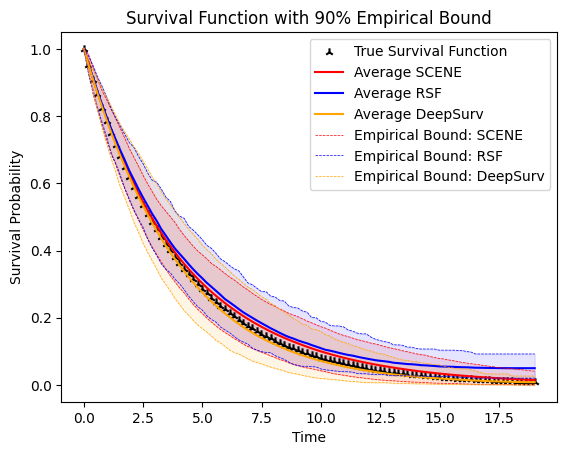}
        \caption{Subject 3: $39.2\%$ Quantile of risk score}
    \end{subfigure}%
    \begin{subfigure}{.5\textwidth}
        \centering
        \includegraphics[width=\linewidth]{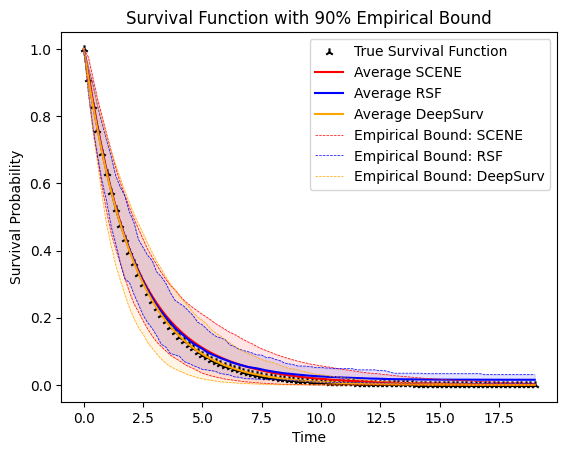}
        \caption{Subject 4: $85.5\%$ Quantile of risk score}
    \end{subfigure}
    \caption{Comparison of conditional survival function estimation for PH Model, C=19, $N=4000$, $d=5$ : $(5\%,95\%)$ empirical bound for Test Subject 1 to Subject 4. }
    % \label{fig:individual-confidence-bands}
\end{figure}

\begin{figure}[!htbp]
    \centering
    \begin{subfigure}{.5\textwidth}
        \centering
        \includegraphics[width=\linewidth]{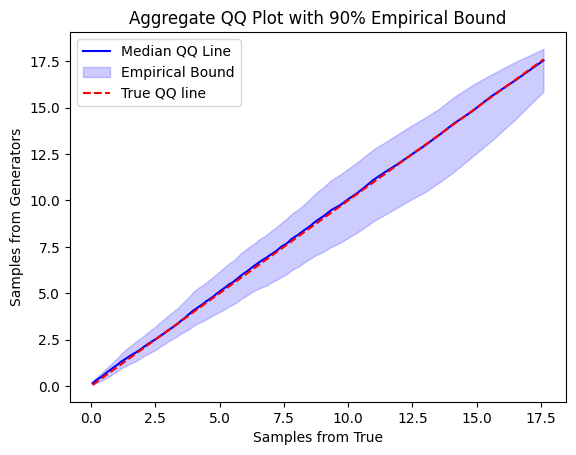}
        \caption{Subject 1: Random sampling}
    \end{subfigure}%
    \begin{subfigure}{.5\textwidth}
        \centering
        \includegraphics[width=\linewidth]{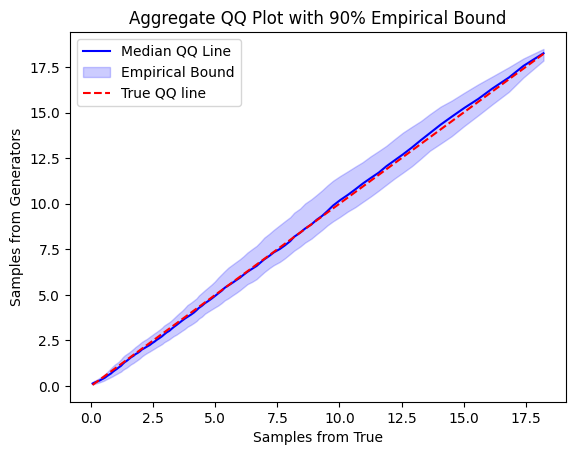}
        \caption{Subject 2: $9.8\%$ Quantile of risk score}
    \end{subfigure}
    \begin{subfigure}{.5\textwidth}
        \centering
        \includegraphics[width=\linewidth]{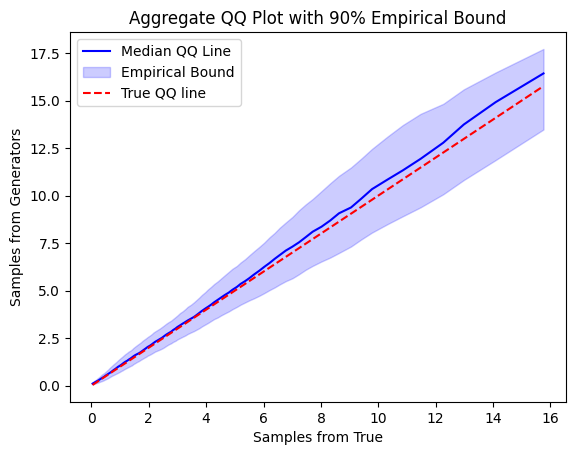}
        \caption{Subject 3: $39.2\%$ Quantile of risk score}
    \end{subfigure}%
    \begin{subfigure}{.5\textwidth}
        \centering
        \includegraphics[width=\linewidth]{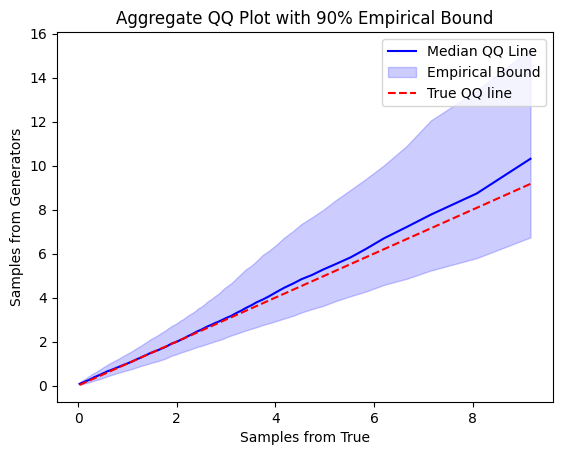}
        \caption{Subject 4: $85.5\%$ Quantile of risk score}
    \end{subfigure}
    \caption{QQ plot of true conditional samples ($x$-axis) and generated samples ($y$-axis) for PH Model, C=19, $N=4000$, $d=5$ : $(5\%,95\%)$ empirical bound for Test Subject 1 to Subject 4.}
\end{figure}

\newpage

\subsection{PO Model / Moderate Censoring / Low dimensional}

\begin{figure}[!htbp]
    \centering
    \begin{subfigure}{.5\textwidth}
        \centering
        \includegraphics[width=\linewidth]{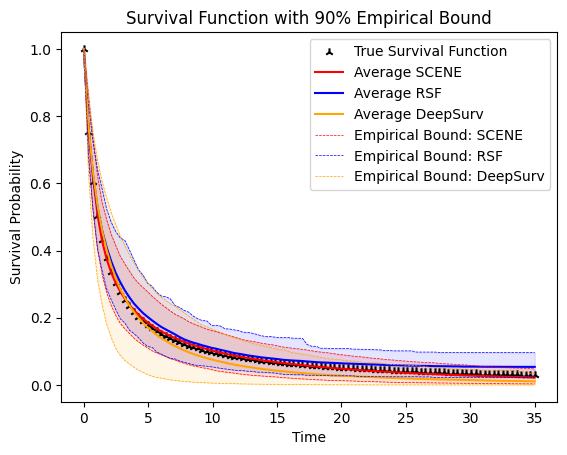}
        \caption{Subject 1: Random sampling}
    \end{subfigure}%
    \begin{subfigure}{.5\textwidth}
        \centering
        \includegraphics[width=\linewidth]{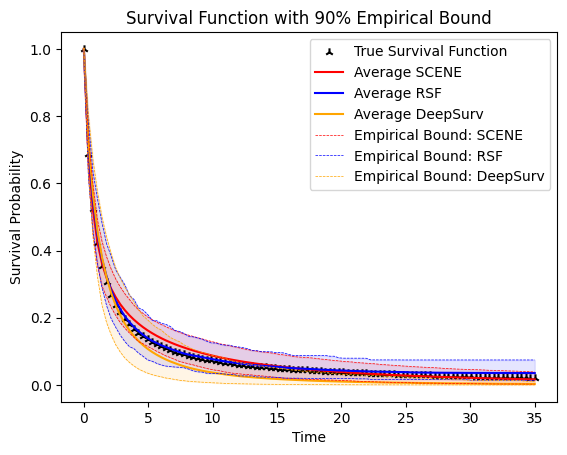}
        \caption{Subject 2: $9.8\%$ Quantile of risk score}
    \end{subfigure}
    \begin{subfigure}{.5\textwidth}
        \centering
        \includegraphics[width=\linewidth]{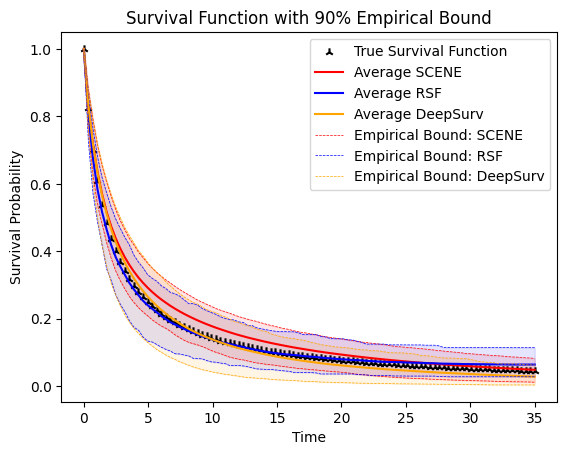}
        \caption{Subject 3: $39.2\%$ Quantile of risk score}
    \end{subfigure}%
    \begin{subfigure}{.5\textwidth}
        \centering
        \includegraphics[width=\linewidth]{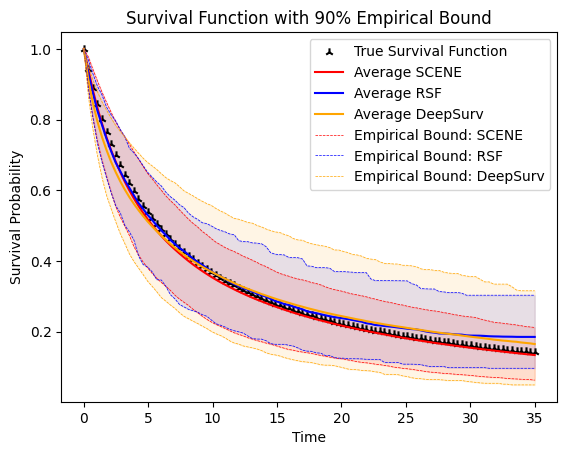}
        \caption{Subject 4: $85.5\%$ Quantile of risk score}
    \end{subfigure}
    
    \caption{Comparison of conditional survival function estimation for PO Model, C=35, $N=4000$, $d=5$ : $(5\%,95\%)$ empirical bound for Test Subject 1 to Subject 4.  }
\end{figure}

\begin{figure}[!htbp]
    \centering
    \begin{subfigure}{.5\textwidth}
        \centering
        \includegraphics[width=\linewidth]{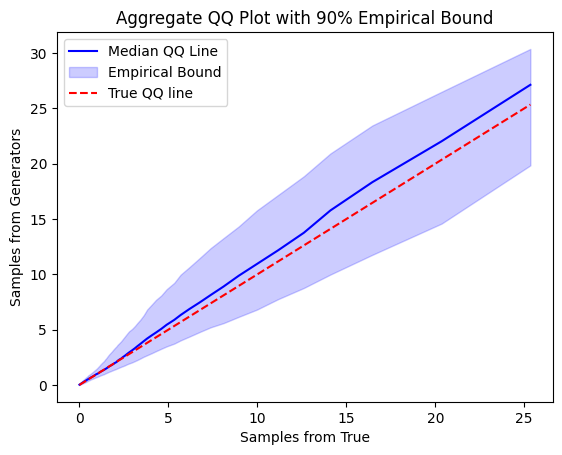}
        \caption{Subject 1: Random sampling}
    \end{subfigure}%
    \begin{subfigure}{.5\textwidth}
        \centering
        \includegraphics[width=\linewidth]{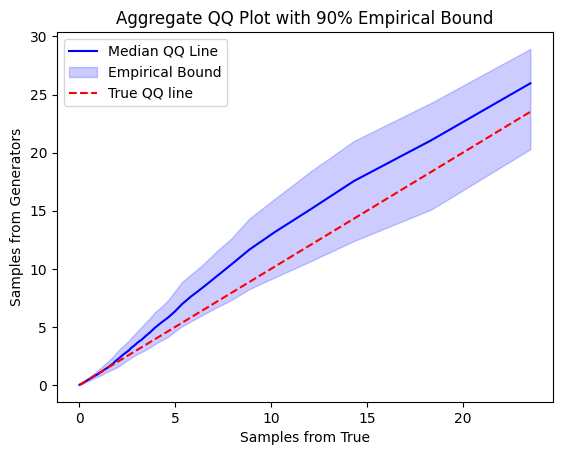}
        \caption{Subject 2: $9.8\%$ Quantile of risk score}
    \end{subfigure}
    \begin{subfigure}{.5\textwidth}
        \centering
        \includegraphics[width=\linewidth]{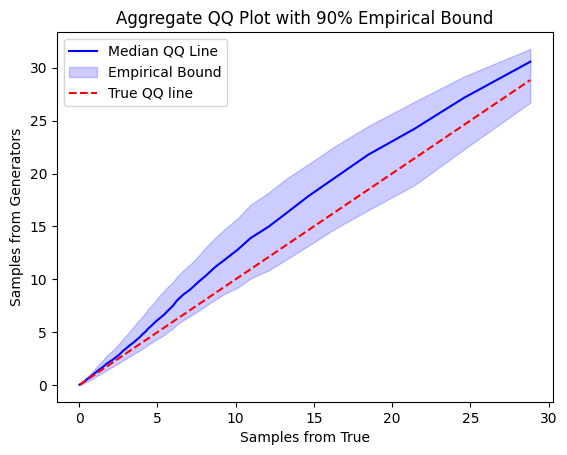}
        \caption{Subject 3: $39.2\%$ Quantile of risk score}
    \end{subfigure}%
    \begin{subfigure}{.5\textwidth}
        \centering
        \includegraphics[width=\linewidth]{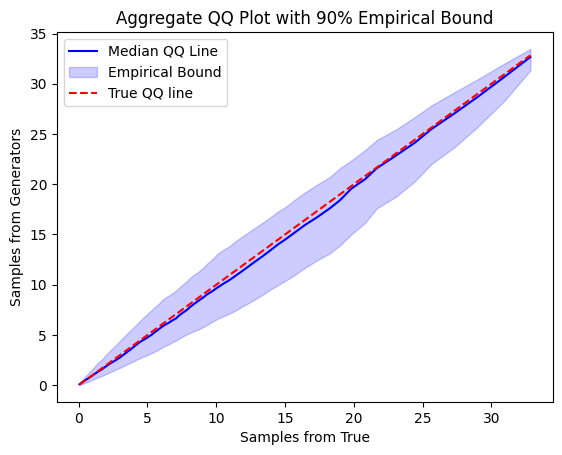}
        \caption{Subject 4: $85.5\%$ Quantile of risk score}
    \end{subfigure}
    \caption{QQ plot of true conditional samples ($x$-axis) and generated samples ($y$-axis) for PO Model, C=35, $N=4000$, $d=5$ : $(5\%,95\%)$ empirical bound for Test Subject 1 to Subject 4. }
    
\end{figure}

\newpage

\subsection{PH Model / Moderate Censoring / High dimensional}

\begin{figure}[!htbp]
    \centering
    \begin{subfigure}{.5\textwidth}
        \centering
        \includegraphics[width=\linewidth]{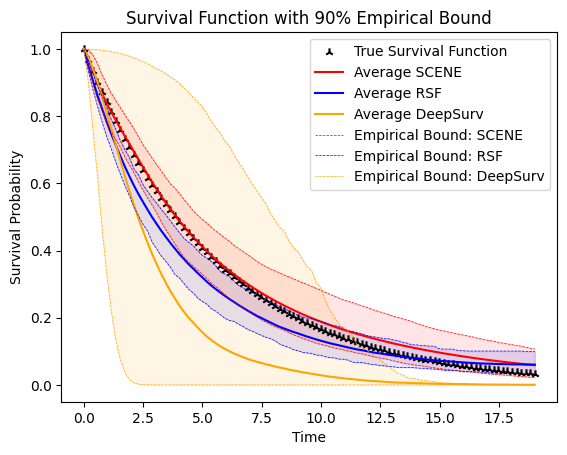}
        \caption{Subject 1: Random sampling}
    \end{subfigure}%
    \begin{subfigure}{.5\textwidth}
        \centering
        \includegraphics[width=\linewidth]{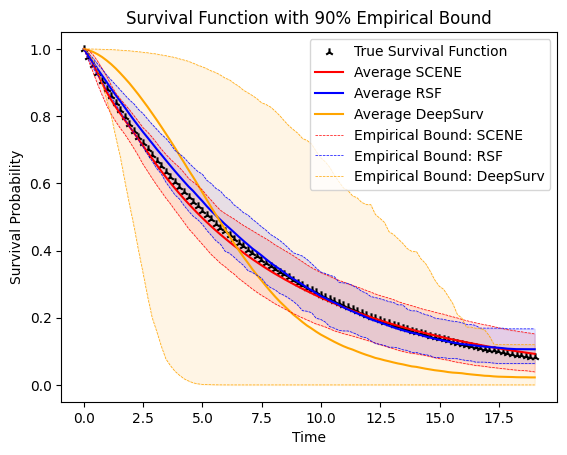}
        \caption{Subject 2: $9.8\%$ Quantile of risk score}
    \end{subfigure}
    \begin{subfigure}{.5\textwidth}
        \centering
        \includegraphics[width=\linewidth]{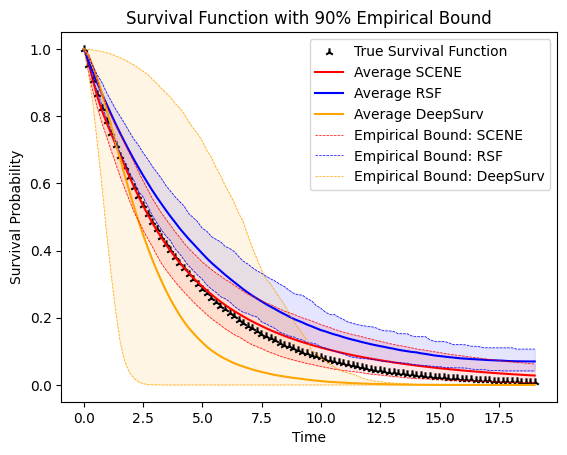}
        \caption{Subject 3: $39.2\%$ Quantile of risk score}
    \end{subfigure}%
    \begin{subfigure}{.5\textwidth}
        \centering
        \includegraphics[width=\linewidth]{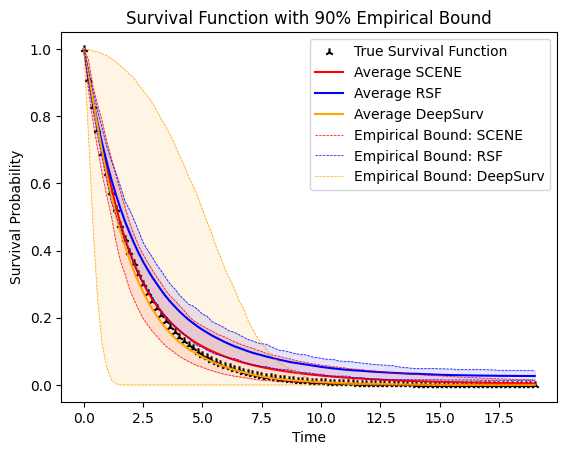}
        \caption{Subject 4: $85.5\%$ Quantile of risk score}
    \end{subfigure}
    \caption{Comparison of conditional survival function estimation for PH Model, C=19, $N=4000$, $d=100$ : $(5\%,95\%)$ empirical bound for Test Subject 1 to Subject 4. }
\end{figure}

\begin{figure}[!htbp]
    \centering
    \begin{subfigure}{.5\textwidth}
        \centering
        \includegraphics[width=\linewidth]{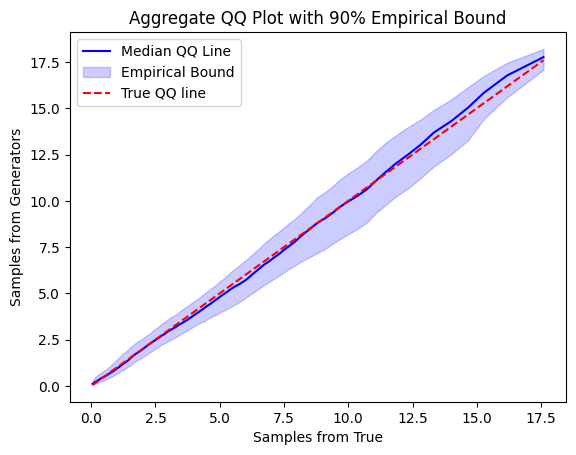}
        \caption{Subject 1: Random sampling}
    \end{subfigure}%
    \begin{subfigure}{.5\textwidth}
        \centering
        \includegraphics[width=\linewidth]{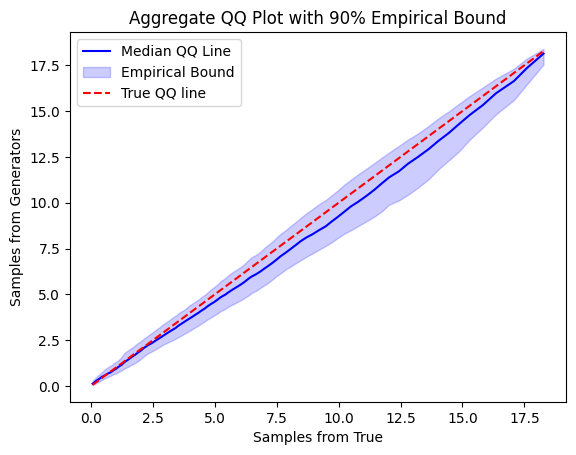}
        \caption{Subject 2: $9.8\%$ Quantile of risk score}
    \end{subfigure}
    \begin{subfigure}{.5\textwidth}
        \centering
        \includegraphics[width=\linewidth]{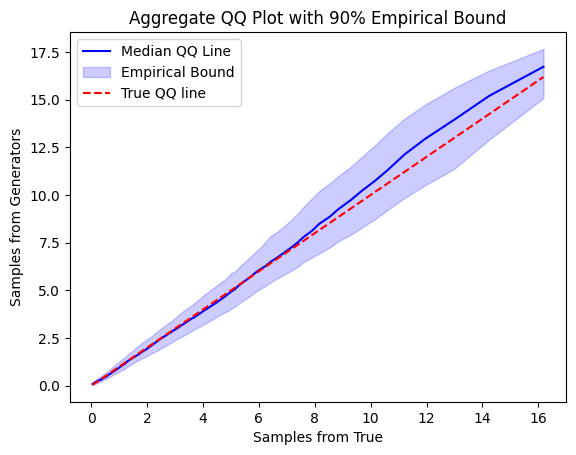}
        \caption{Subject 3: $39.2\%$ Quantile of risk score}
    \end{subfigure}%
    \begin{subfigure}{.5\textwidth}
        \centering
        \includegraphics[width=\linewidth]{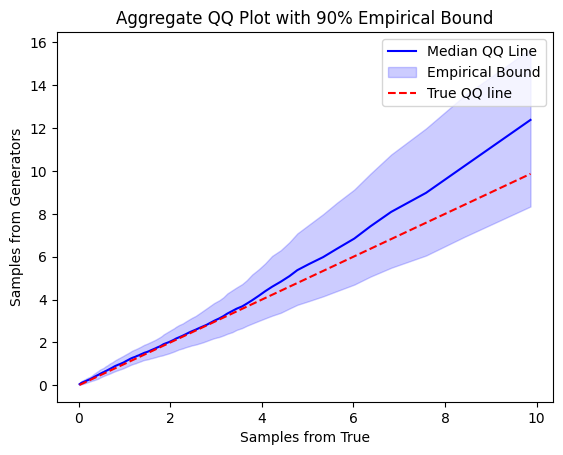}
        \caption{Subject 4: $85.5\%$ Quantile of risk score}
    \end{subfigure}
    \caption{QQ plot of true conditional samples ($x$-axis) and generated samples ($y$-axis) for PH Model, C=19, $N=4000$, $d=100$ : $(5\%,95\%)$ empirical bound for Test Subject 1 to Subject 4.}
    \label{fig:individual-confidence-bands}
\end{figure}

\newpage

\subsection{PO Model/ Moderate Censoring / High dimensional}

\begin{figure}[!htbp]
    \centering
    \begin{subfigure}{.5\textwidth}
        \centering
        \includegraphics[width=\linewidth]{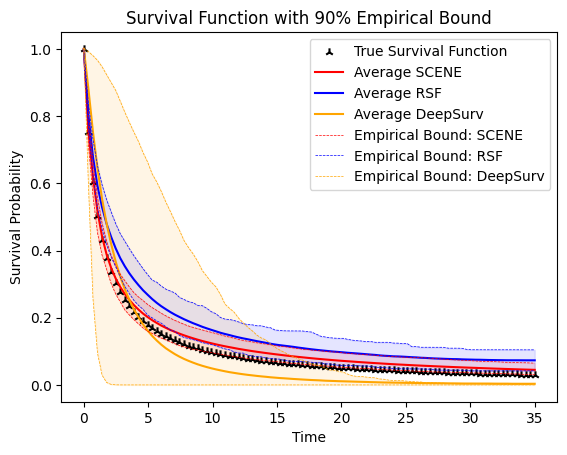}
        \caption{Subject 1: Random sampling}
    \end{subfigure}%
    \begin{subfigure}{.5\textwidth}
        \centering
        \includegraphics[width=\linewidth]{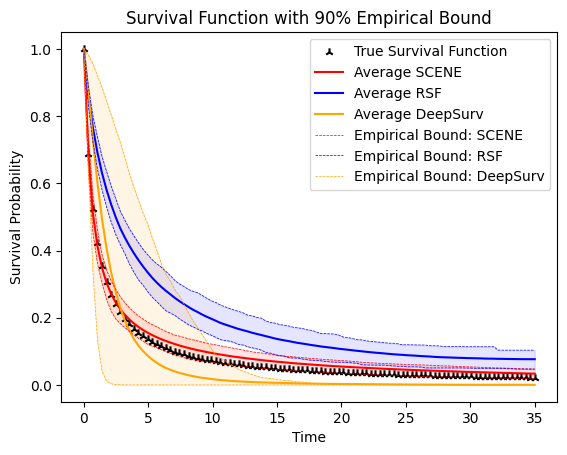}
        \caption{Subject 2: $9.8\%$ Quantile of risk score}
    \end{subfigure}
    \begin{subfigure}{.5\textwidth}
        \centering
        \includegraphics[width=\linewidth]{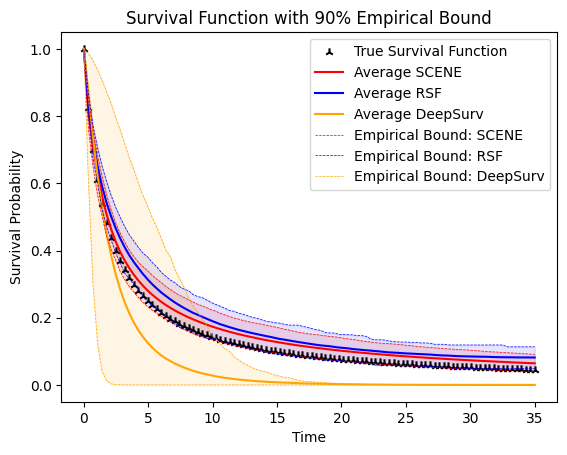}
        \caption{Subject 3: $39.2\%$ Quantile of risk score}
    \end{subfigure}%
    \begin{subfigure}{.5\textwidth}
        \centering
        \includegraphics[width=\linewidth]{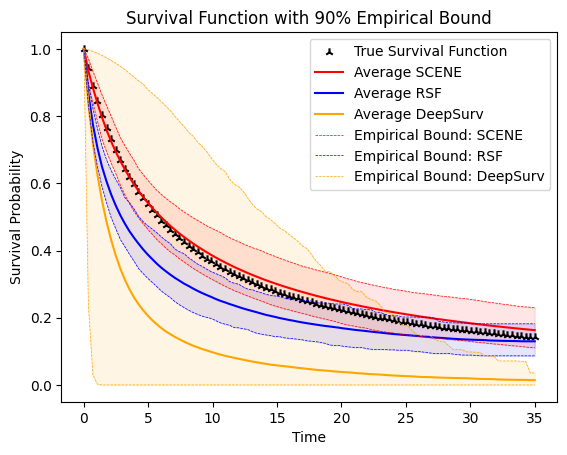}
        \caption{Subject 4: $85.5\%$ Quantile of risk score}
    \end{subfigure}
    \caption{PO Model, C=35, $N=4000$, $d=100$ : $(5\%,95\%)$ empirical bound for Test Subject 1 to Subject 4.}
    \label{fig:individual-confidence-bands}
\end{figure}

\begin{figure}[!htbp]
    \centering
    \begin{subfigure}{.5\textwidth}
        \centering
        \includegraphics[width=\linewidth]{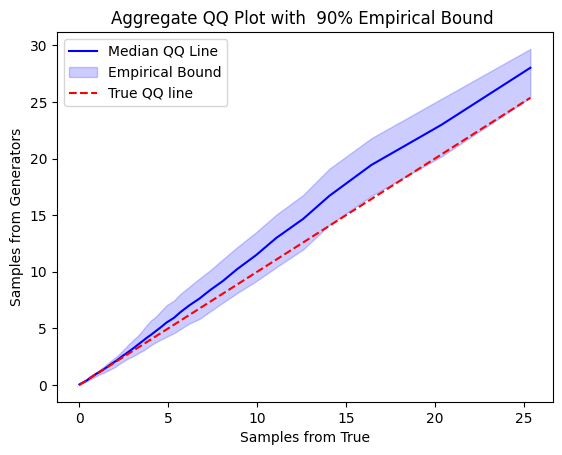}
        \caption{Subject 1: Random sampling}
    \end{subfigure}%
    \begin{subfigure}{.5\textwidth}
        \centering
        \includegraphics[width=\linewidth]{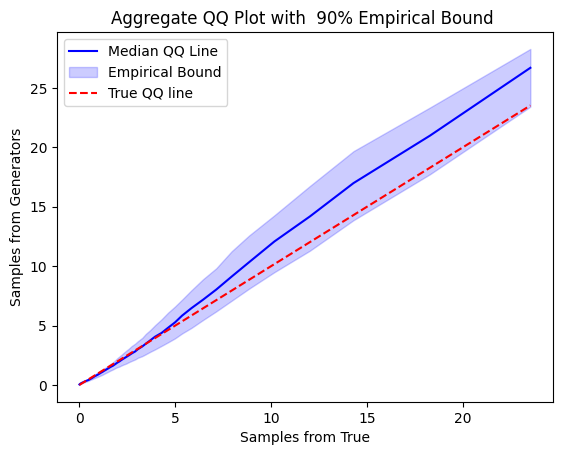}
        \caption{Subject 2: $9.8\%$ Quantile of risk score}
    \end{subfigure}
    \begin{subfigure}{.5\textwidth}
        \centering
        \includegraphics[width=\linewidth]{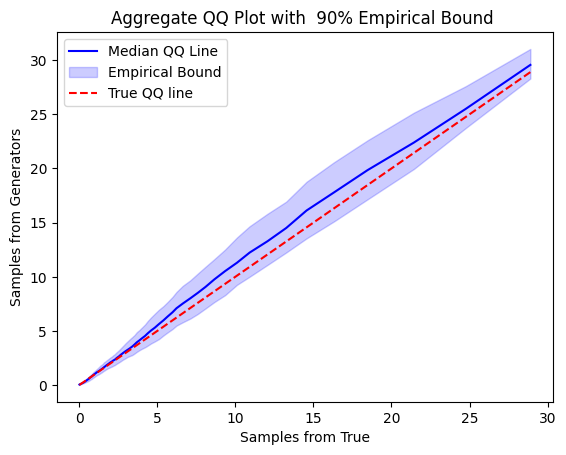}
        \caption{Subject 3: $39.2\%$ Quantile of risk score}
    \end{subfigure}%
    \begin{subfigure}{.5\textwidth}
        \centering
        \includegraphics[width=\linewidth]{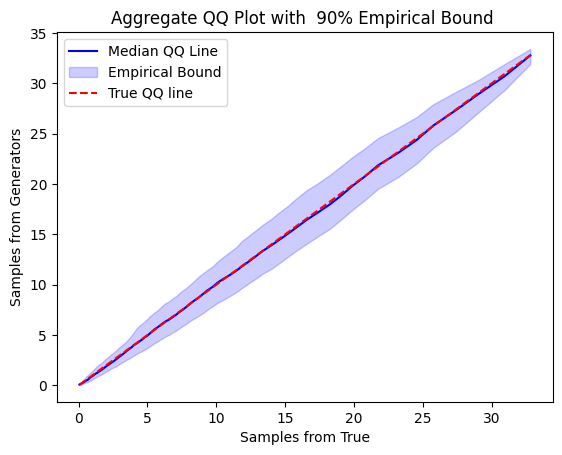}
        \caption{Subject 4: $85.5\%$ Quantile of risk score}
    \end{subfigure}
    \caption{PO Model, C=35, $N=4000$, $d=100$ : $(5\%,95\%)$ empirical bound for Test Subject 1 to Subject 4.}
    \label{fig:individual-confidence-bands}
\end{figure}

\end{document}